\def\year{2021}\relax
\documentclass[letterpaper]{article} 
\usepackage{aaai21}  
\usepackage{times}  
\usepackage{helvet} 
\usepackage{courier}  
\usepackage[hyphens]{url}  
\usepackage{graphicx} 
\usepackage{graphicx}  
\usepackage{natbib}  
\usepackage{caption} 
\usepackage[switch]{lineno}
\usepackage{bm}
\usepackage{amsmath}
\usepackage{amssymb}
\usepackage{amsthm}
\usepackage{subfigure}
\usepackage{multirow}

\usepackage{algorithm}
\usepackage{algorithmic}
\usepackage{url}
\newtheorem{theorem}{Theorem}

\newtheorem{definition}[theorem]{Definition}

\title{Learning Graphons via Structured Gromov-Wasserstein Barycenters}
\author{
    Hongteng Xu\textsuperscript{\rm 1, 2}, 
    Dixin Luo\textsuperscript{\rm 3}\thanks{Correspondence author}, 
    Lawrence Carin\textsuperscript{\rm 4}, 
    Hongyuan Zha\textsuperscript{\rm 5}\\
}
\affiliations{
	\textsuperscript{\rm 1}Gaoling School of Artificial Intelligence, Renmin University of China,\\
	\textsuperscript{\rm 2} Beijing Key Laboratory of Big Data Management and Analysis Methods,\\
	\textsuperscript{\rm 3}School of Computer Science and Technology, Beijing Institue of Technology, 
	\textsuperscript{\rm 4}Department of ECE, Duke University,\\
	\textsuperscript{\rm 5}School of Data Science, Shenzhen Research Institute of Big Data, The Chinese University of Hong Kong, Shenzhen\footnote{Currently on leave from College of Computing, Georgia Institute of Technology.}

}
\usepackage{color}
\usepackage{xcolor}

\begin{document}
\maketitle

\begin{abstract}
We propose a novel and principled method to learn a nonparametric graph model called \textit{graphon}, which is defined in an infinite-dimensional space and represents arbitrary-size graphs. 
Based on the weak regularity lemma from the theory of graphons, we leverage a step function to approximate a graphon. 
We show that the cut distance of graphons can be relaxed to the Gromov-Wasserstein distance of their step functions. 
Accordingly, given a set of graphs generated by an underlying graphon, we learn the corresponding step function as the Gromov-Wasserstein barycenter of the given graphs.
Furthermore, we develop several enhancements and extensions of the basic algorithm, $e.g.$, the smoothed Gromov-Wasserstein barycenter for guaranteeing the continuity of the learned graphons and the mixed Gromov-Wasserstein barycenters for learning multiple structured graphons. 
The proposed approach overcomes drawbacks of prior state-of-the-art methods, and outperforms them on both synthetic and real-world data. 
The code is available at \url{https://github.com/HongtengXu/SGWB-Graphon}.
\end{abstract}

\section{Introduction}
Given a set of graphs, $e.g.$, social networks and biological networks, we are often interested in modeling their generative mechanisms and building statistical graph models~\cite{kolaczyk2009statistical,goldenberg2010survey}. 
Many efforts have been made to achieve this aim, leading to such methods as the stochastic block model~\cite{nowicki2001estimation}, the graphlet~\cite{soufiani2012graphlet}, and the latent space model~\cite{hoff2002latent}. 
However, when dealing with large-scale complex networks, the parametric models above are often oversimplified, and thus, suffer from underfitting. 
To enhance the model capacity, a nonparametric graph model called \textit{graphon} (or \textit{graph limit}) was proposed~\cite{janson2008graph,lovasz2012large}. 
Mathematically, a graphon is a two-dimensional symmetric Lebesgue measurable function, denoted as $W: \Omega^2\mapsto [0, 1]$, where $\Omega$ is a measure space, $e.g.$, $\Omega=[0, 1]$. 
Given a graphon, we can generate arbitrarily sized graphs by the following sampling process:
\begin{eqnarray}\label{eq:generate_graph}
\begin{aligned}
&v_n&&\sim \text{Uniform}(\Omega),~\text{for}~n=1,...,N,\\
&a_{nn'}&&\sim \text{Bernoulli}(W(v_n, v_{n'})),~\text{for}~n,n'=1,...,N.
\end{aligned}
\end{eqnarray}
The first step samples $N$ nodes independently from a uniform distribution defined on $\Omega$. 
The second step generates an adjacency matrix $\bm{A}=[a_{nn'}]\in \{0, 1\}^{N\times N}$, whose elements yield the Bernoulli distributions determined by the graphon. 
Accordingly, we derive a graph $\mathcal{G}(\mathcal{V},\mathcal{E})$ with $\mathcal{V}=\{1,...,N\}$ and $\mathcal{E}=\{(n,n')~|~a_{nn'}=1\}$. 

This graphon model is useful theoretically to characterize complex graphs~\cite{chung2011spectra,lovasz2012large}, which has been widely used in many applications, $e.g.$, network centrality~\cite{ballester2006who,avella2018centrality}, control~\cite{jackson2015games,gao2019graphon}, and optimization~\cite{nagurney2013network,parise2018graphon}. 
A fundamental problem connected to these applications concerns \textit{how to robustly learn graphons from observed graphs.} 

Many learning methods have been developed to solve this problem. 
Most of them are based on the weak regularity lemma of graphon~\cite{frieze1999quick}. 
This lemma indicates that an arbitrary graphon can be approximated well by a two-dimensional step function. 
To learn step functions as target graphons, existing methods either leverage stochastic block models, $e.g.$, the sorting-and-smoothing (SAS) method~\cite{chan2014consistent}, the stochastic block approximation (SBA)~\cite{airoldi2013stochastic}, and its variant ``largest gap'' (LG)~\cite{channarond2012classification}, or they apply low-rank approximation directly to observed graphs, $e.g.$, the matrix completion (MC) method~\cite{keshavan2010matrix} and the universal singular value thresholding (USVT) algorithm~\cite{chatterjee2015matrix}. 

\begin{figure*}[t!]
    \centering
    \subfigure[$W(x, y)=xy$ and $0\leq x, y\leq 1$]{
    \includegraphics[height=1.9cm]{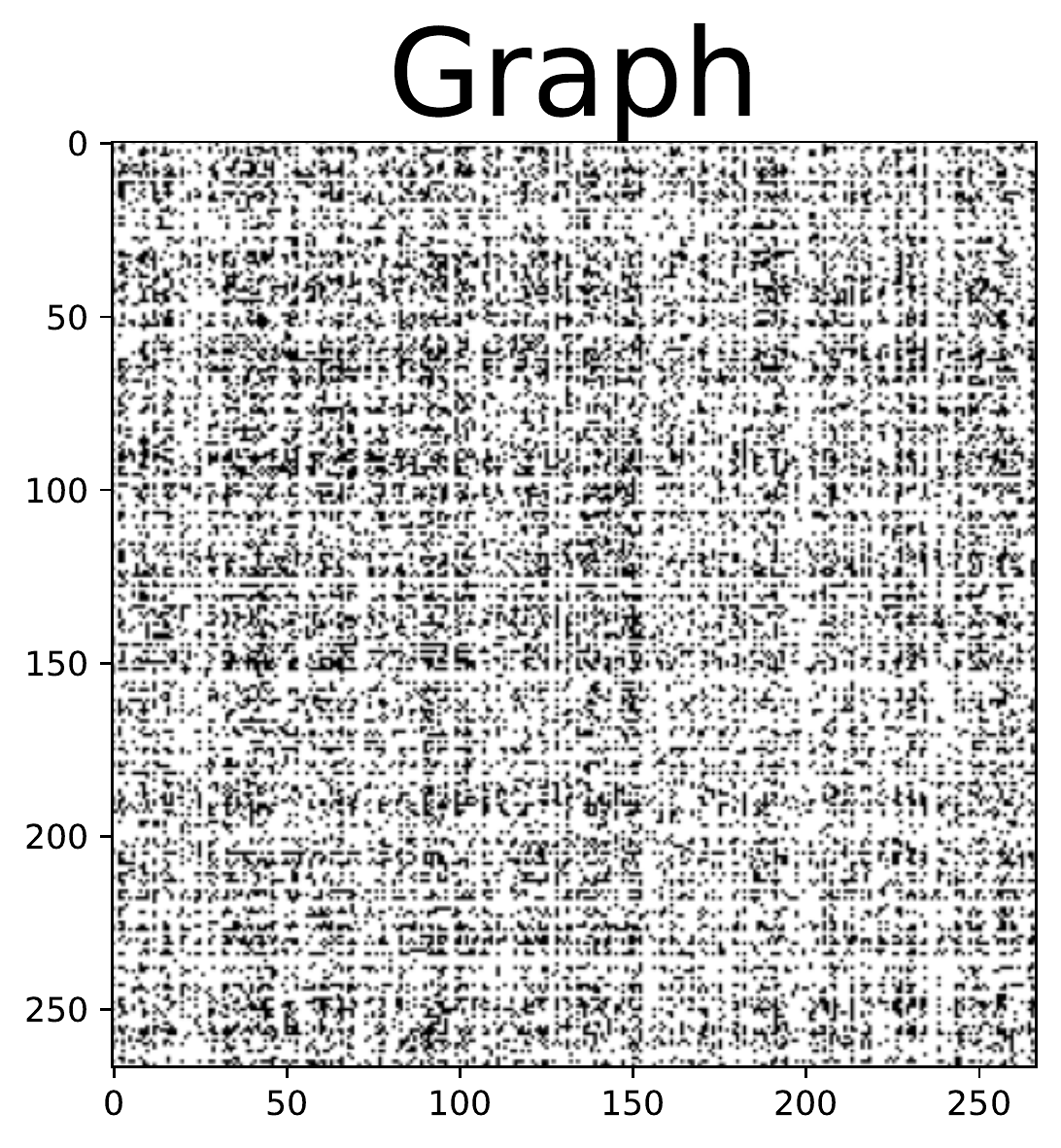}
    \includegraphics[height=1.9cm]{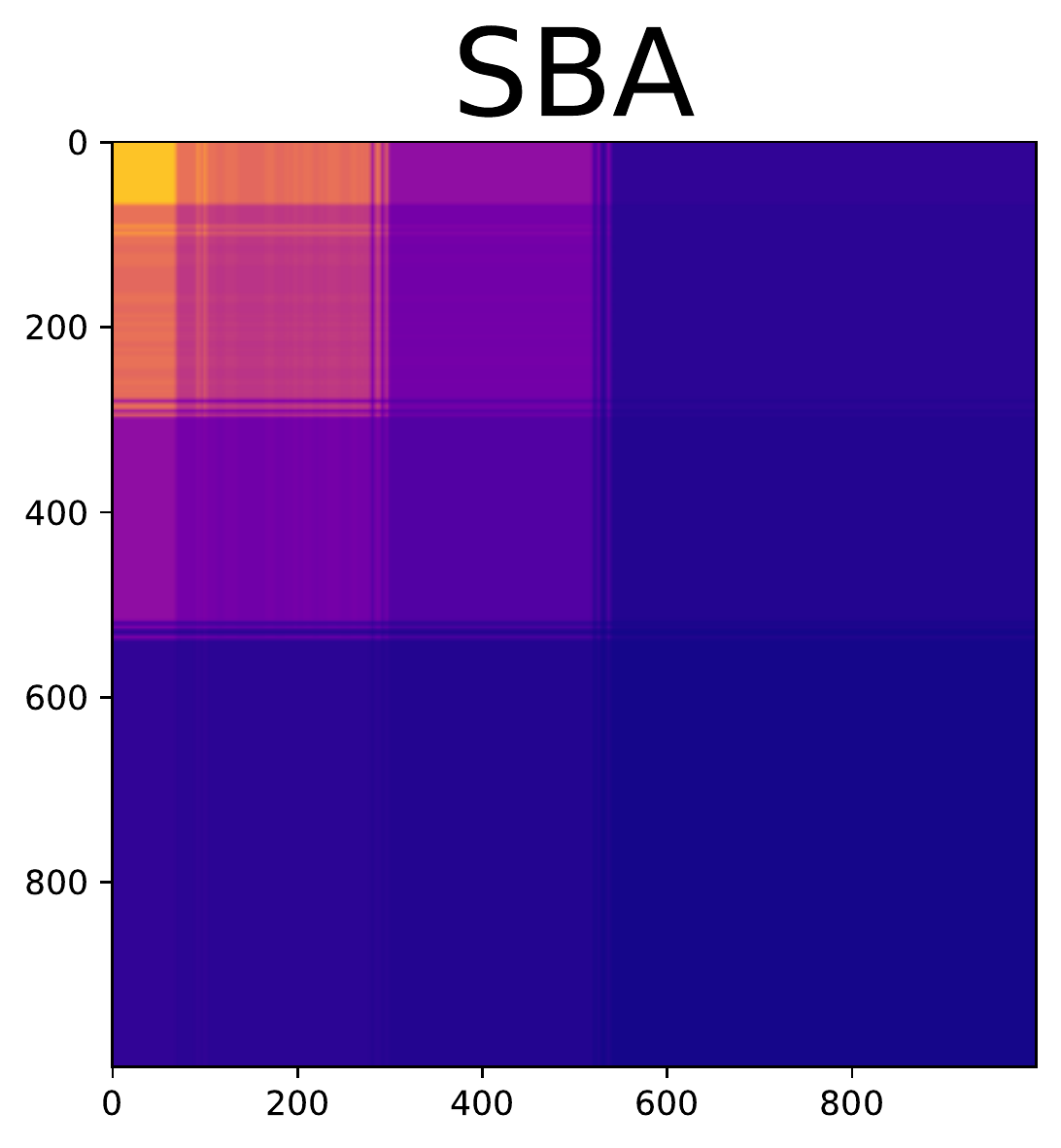}
    \includegraphics[height=1.9cm]{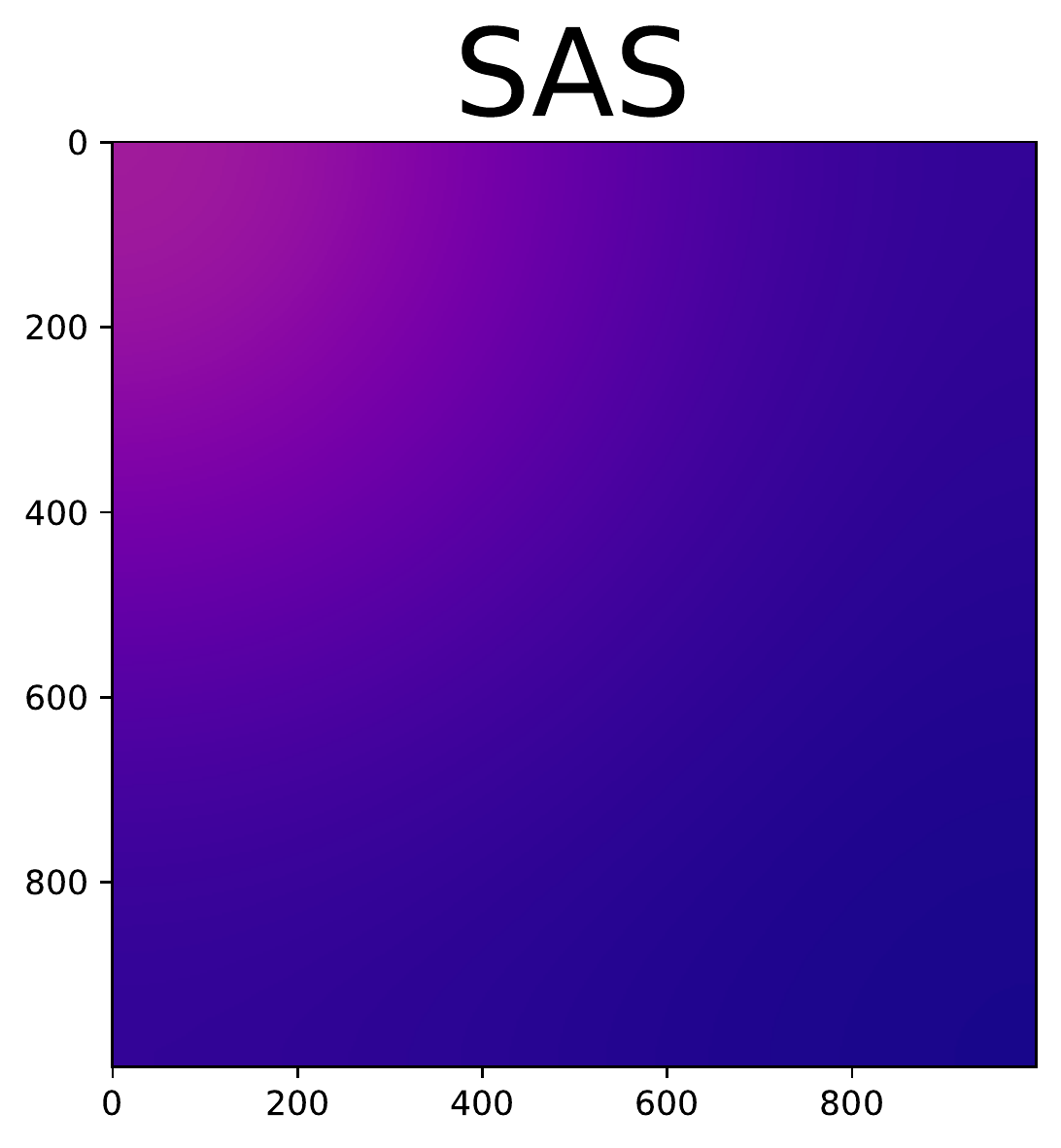}
    \includegraphics[height=1.9cm]{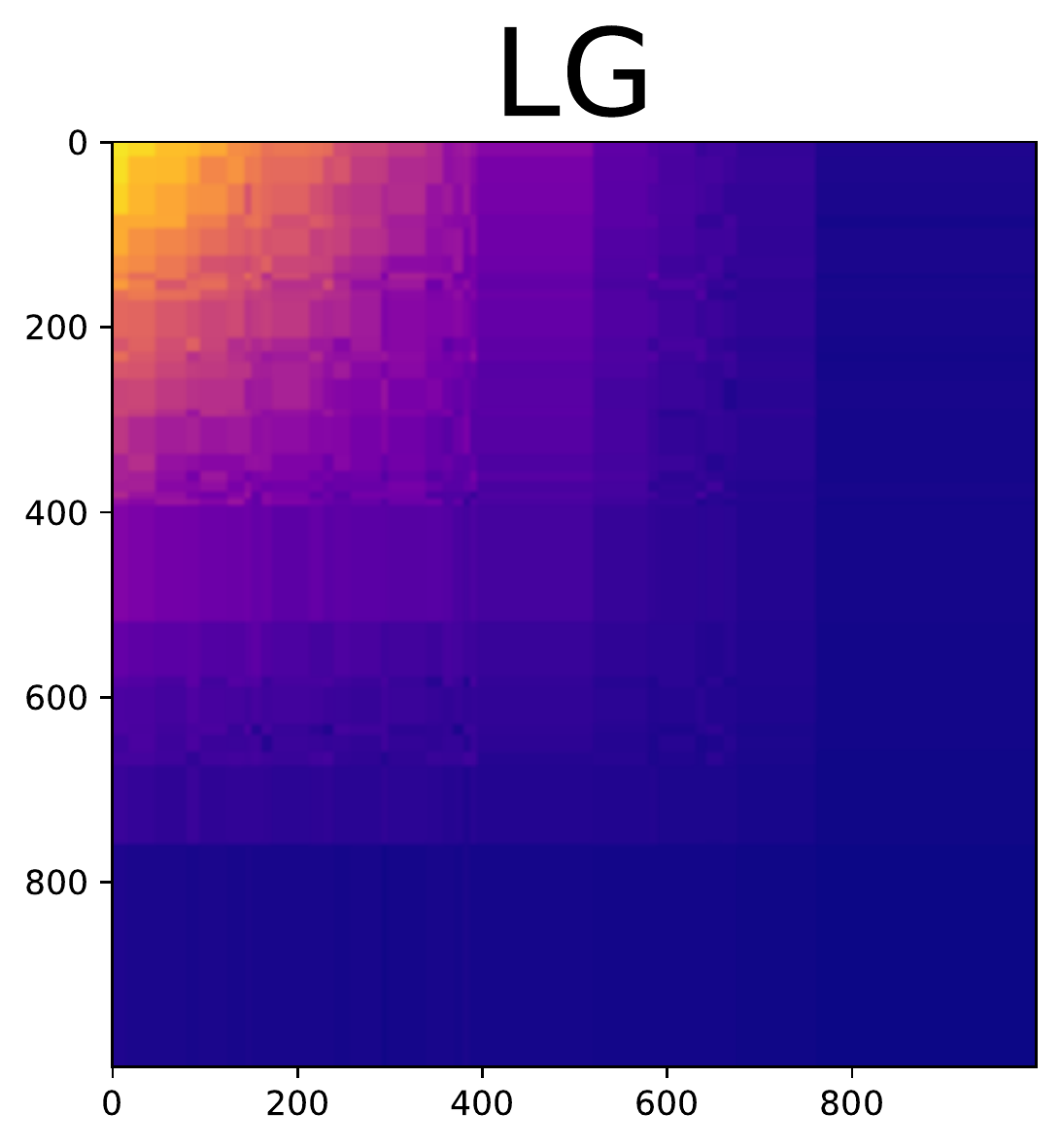}
    \includegraphics[height=1.9cm]{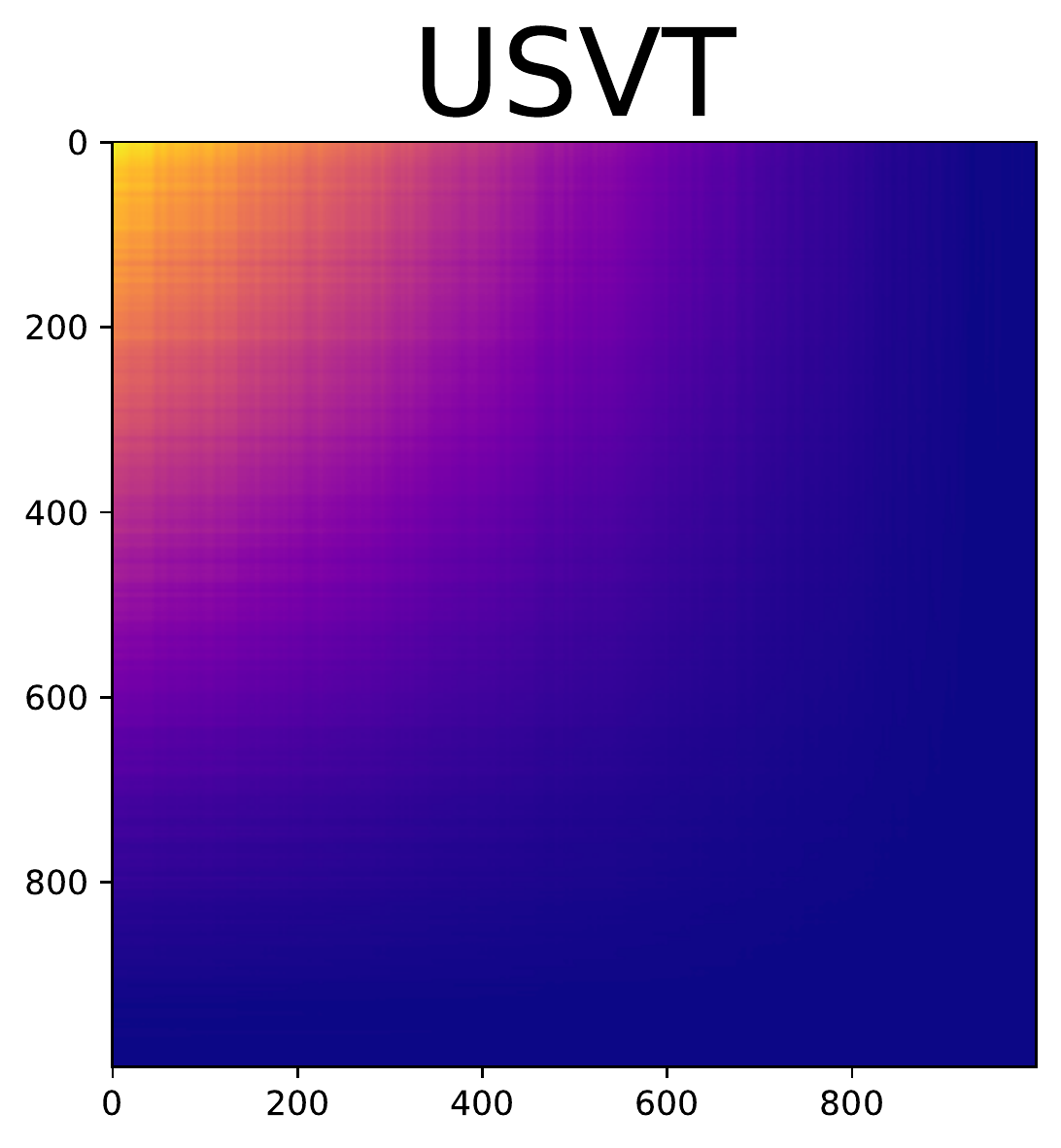}
    \includegraphics[height=1.9cm]{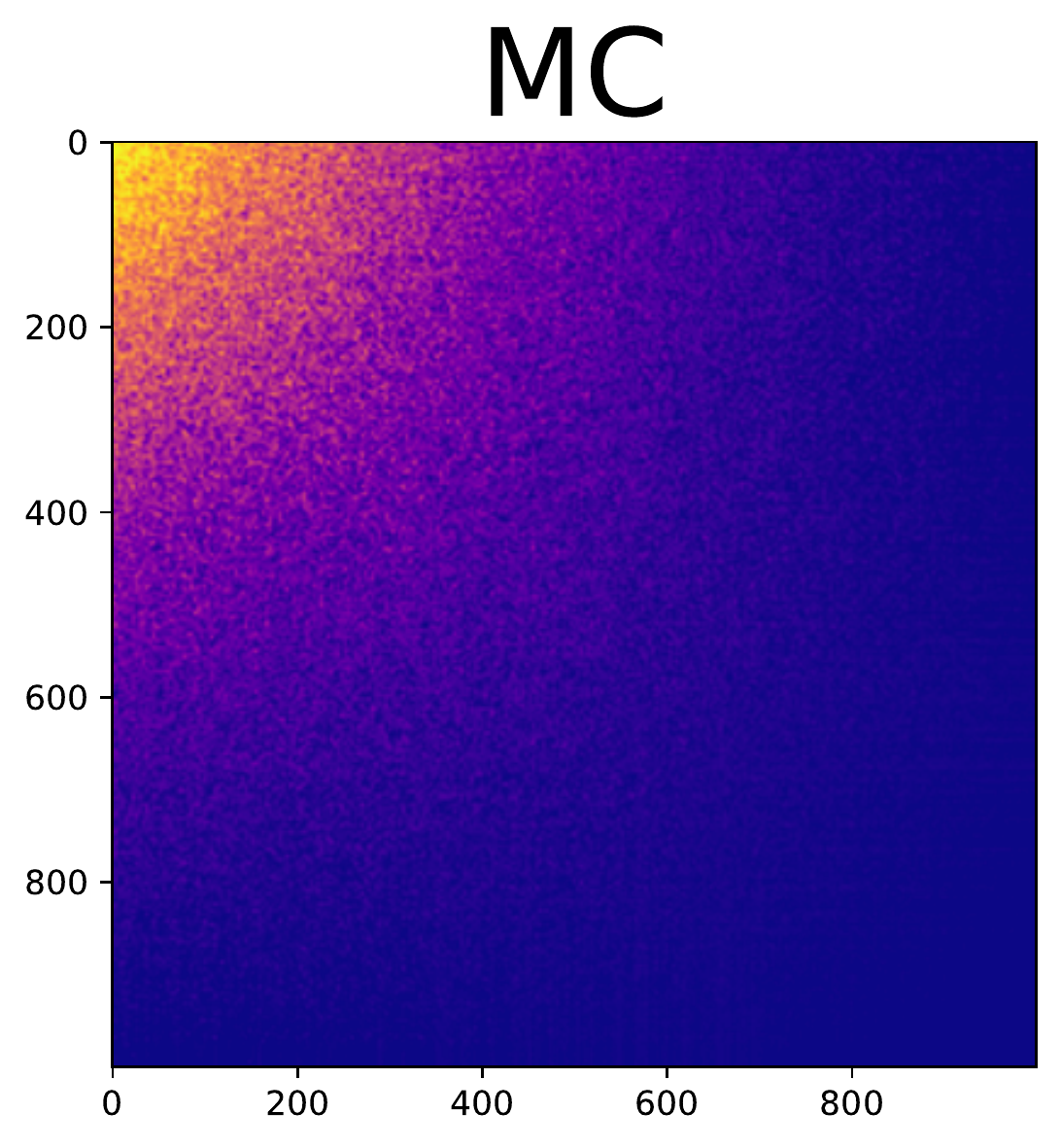}
    \includegraphics[height=1.9cm]{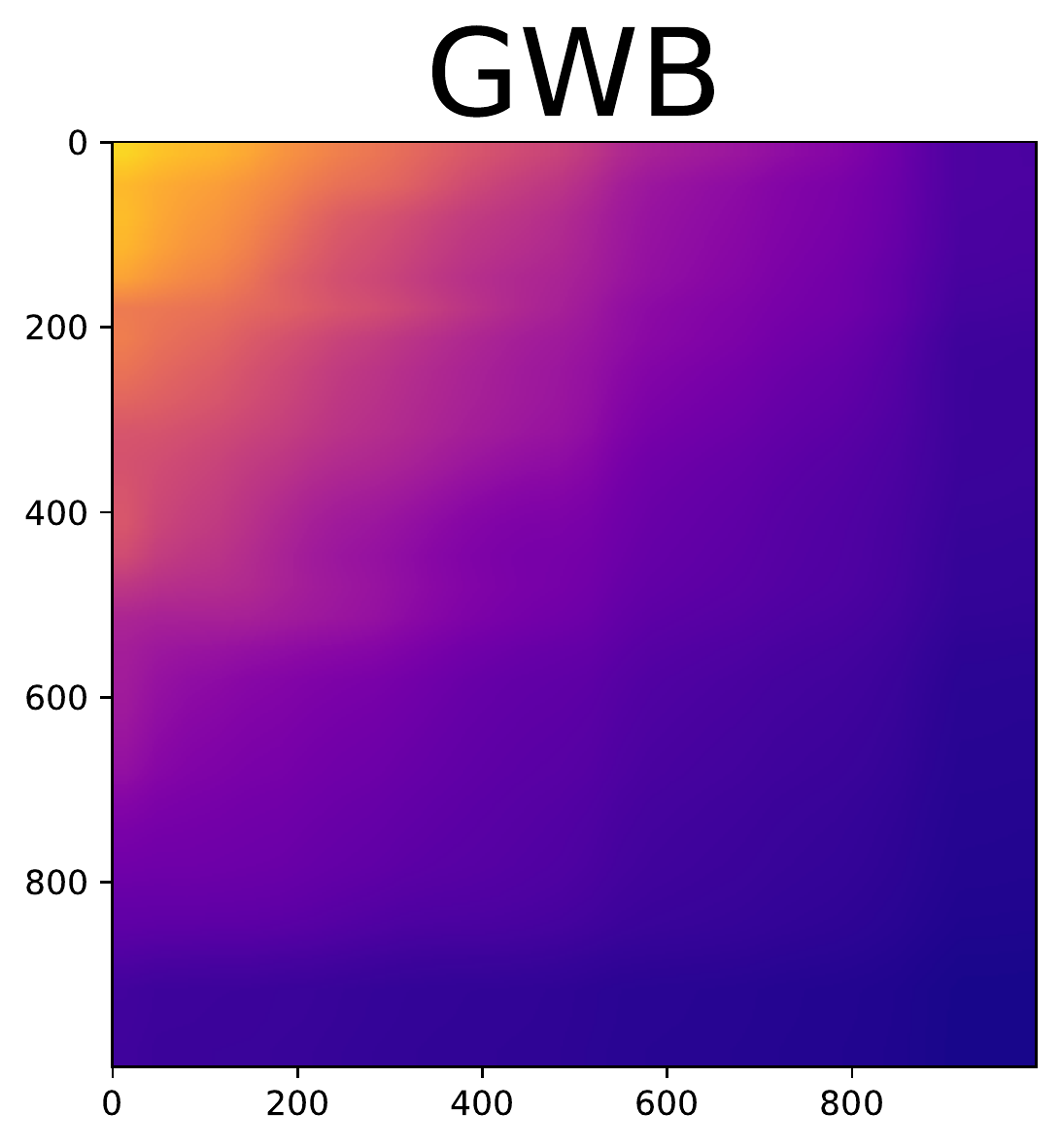}
    \includegraphics[height=1.9cm]{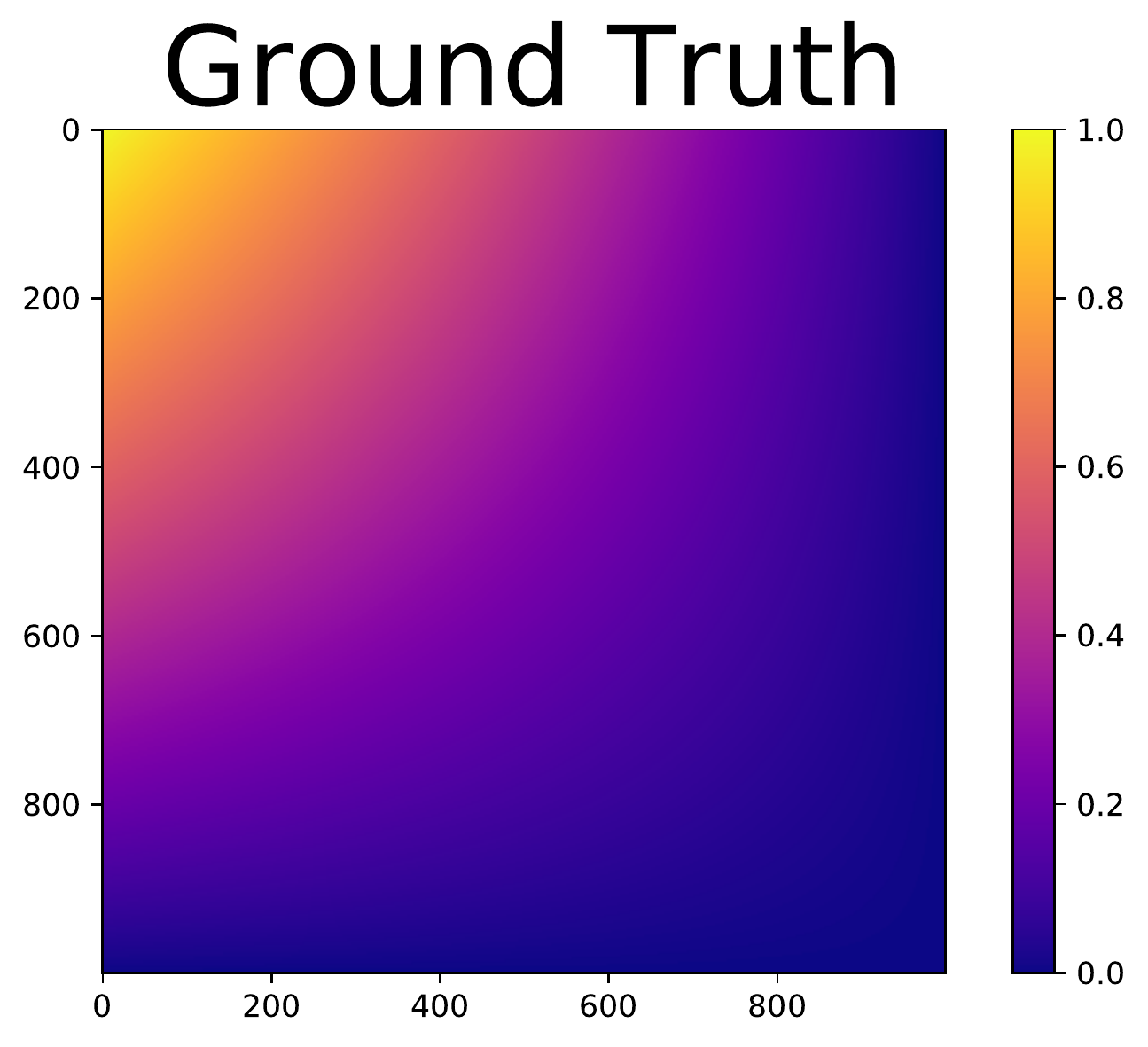}\label{fig:example1a}
    }\\
    \vspace{-4mm}
    \subfigure[$W(x, y)=|x - y|$ and $0\leq x, y\leq 1$]{
    \includegraphics[height=1.9cm]{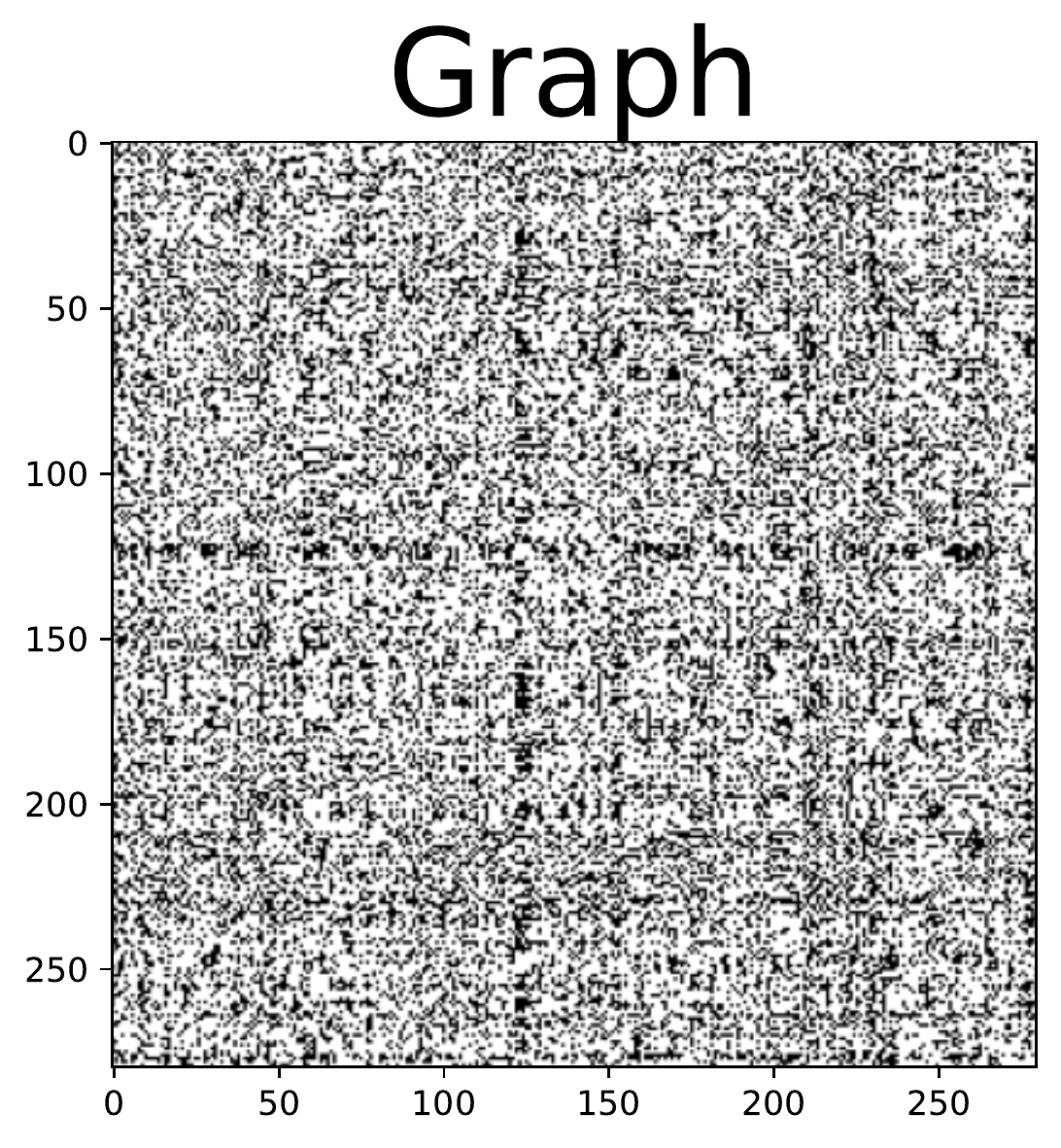}
    \includegraphics[height=1.9cm]{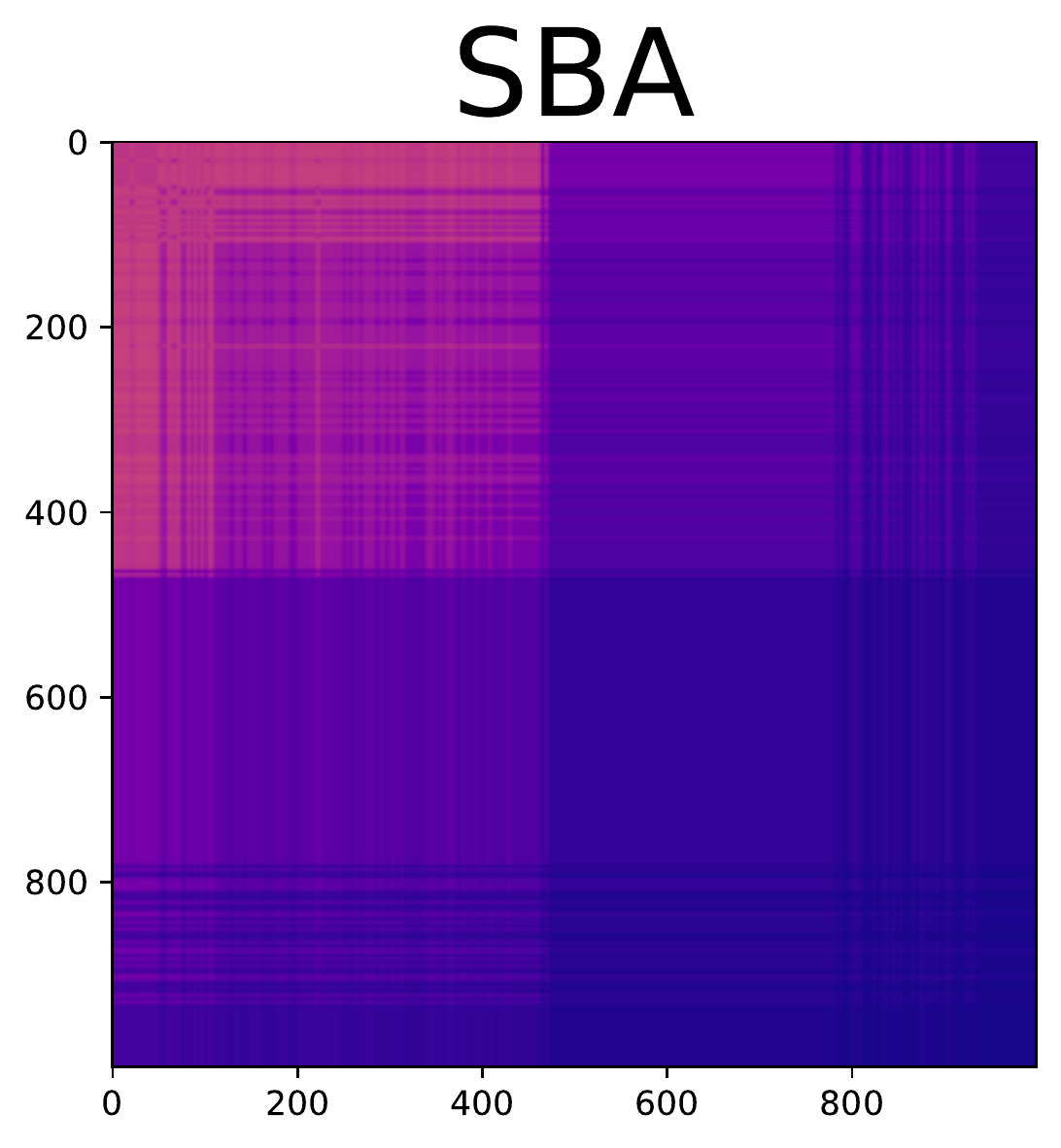}
    \includegraphics[height=1.9cm]{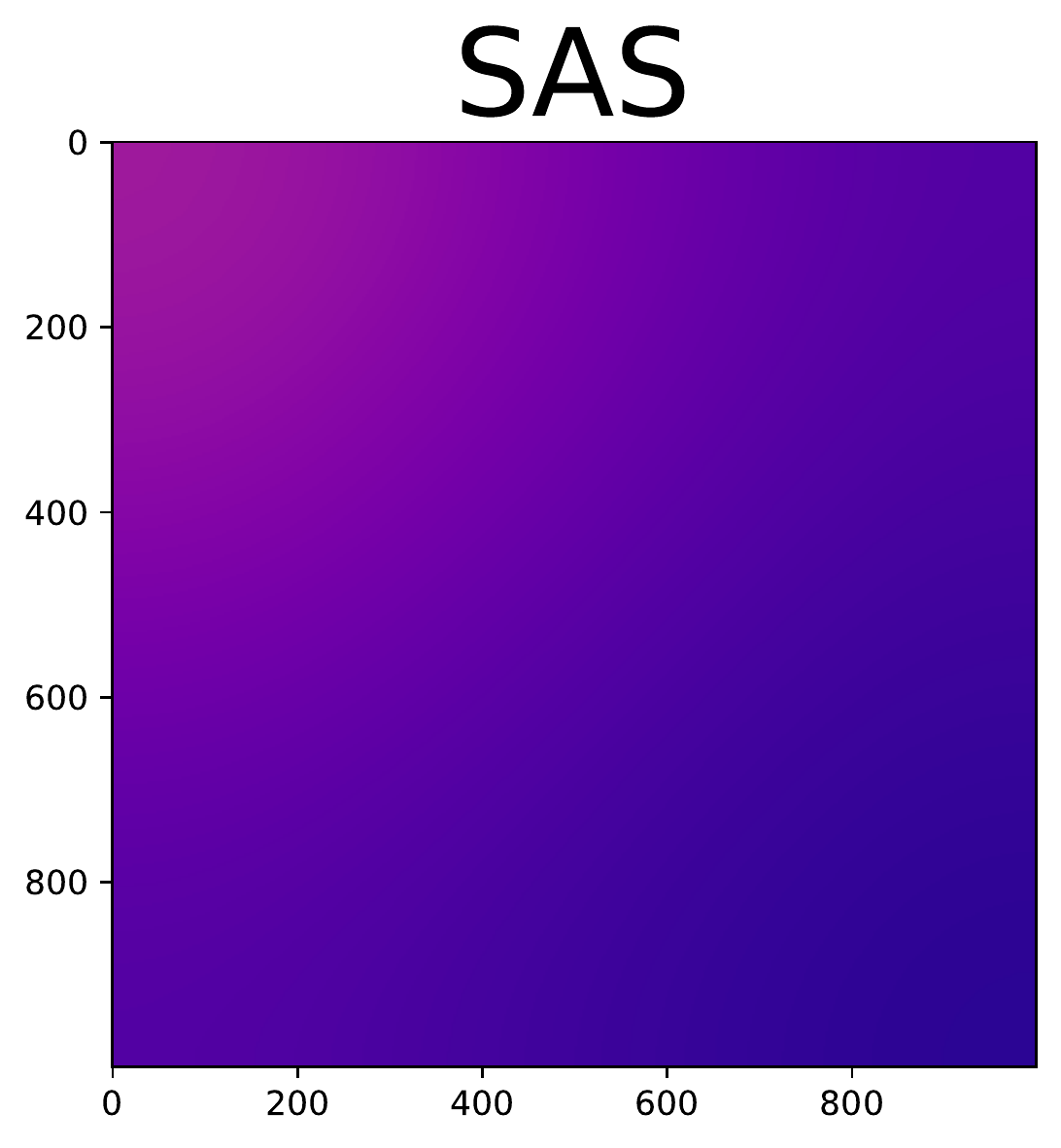}
    \includegraphics[height=1.9cm]{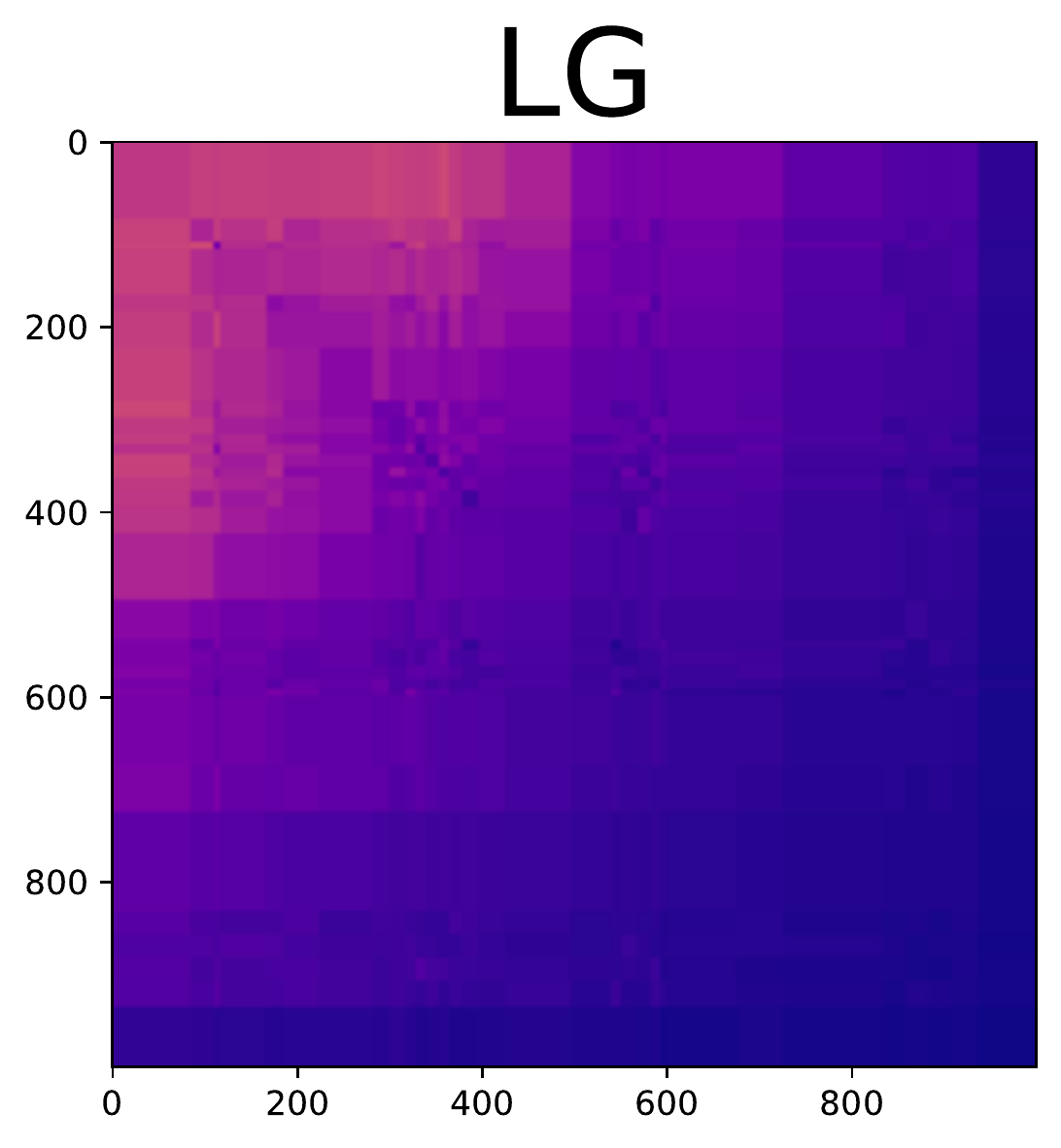}
    \includegraphics[height=1.9cm]{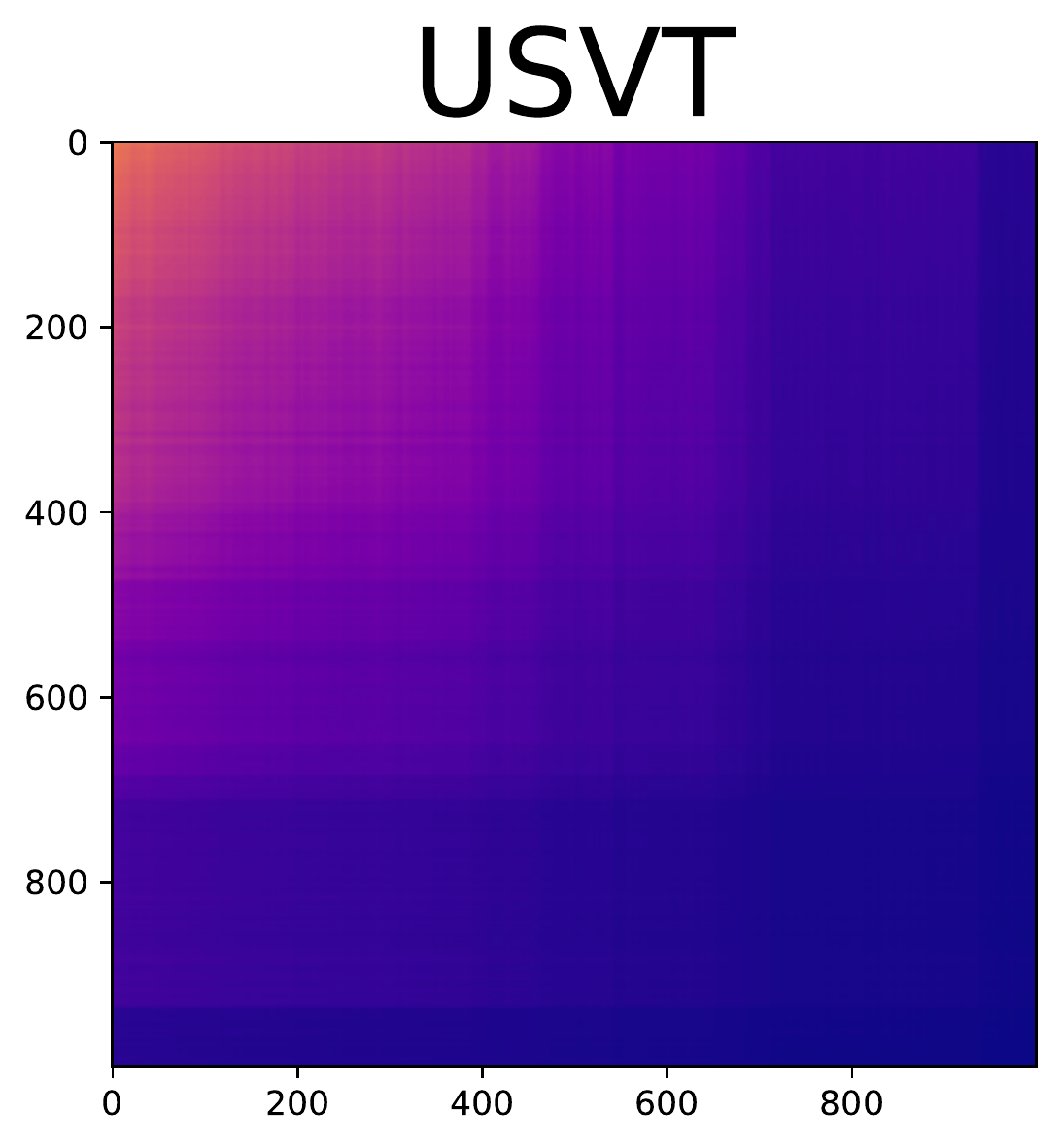}
    \includegraphics[height=1.9cm]{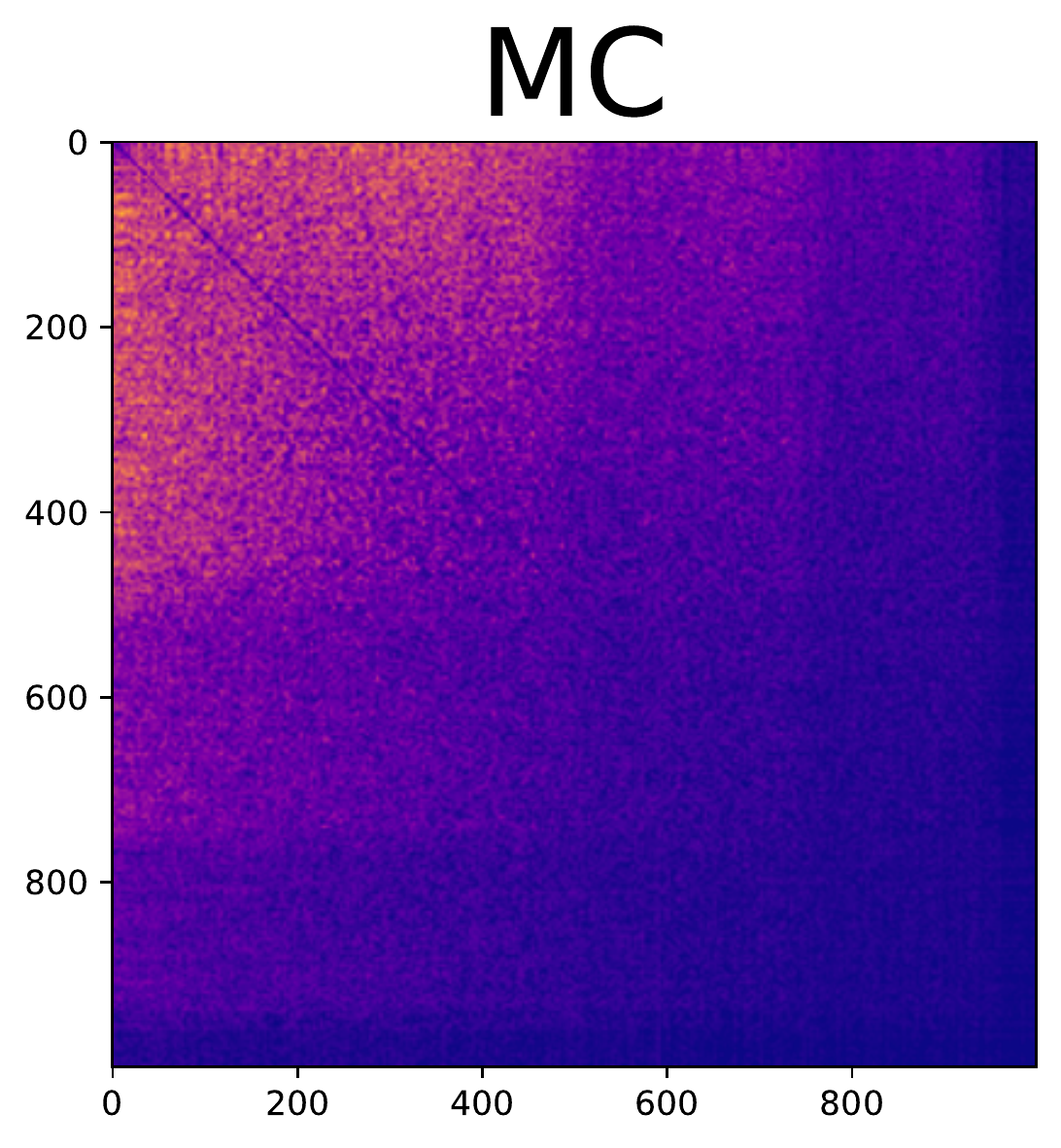}
    \includegraphics[height=1.9cm]{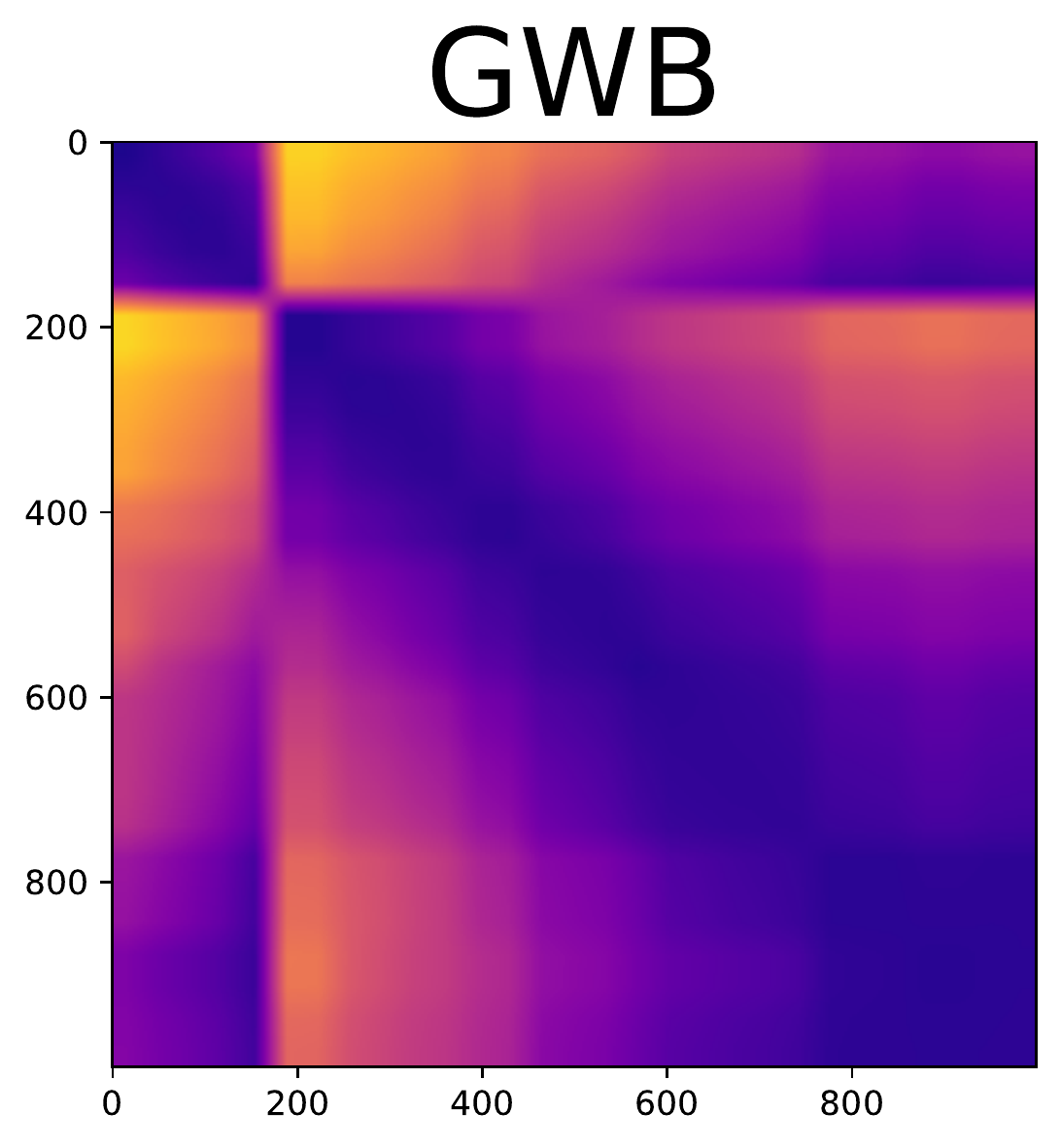}
    \includegraphics[height=1.9cm]{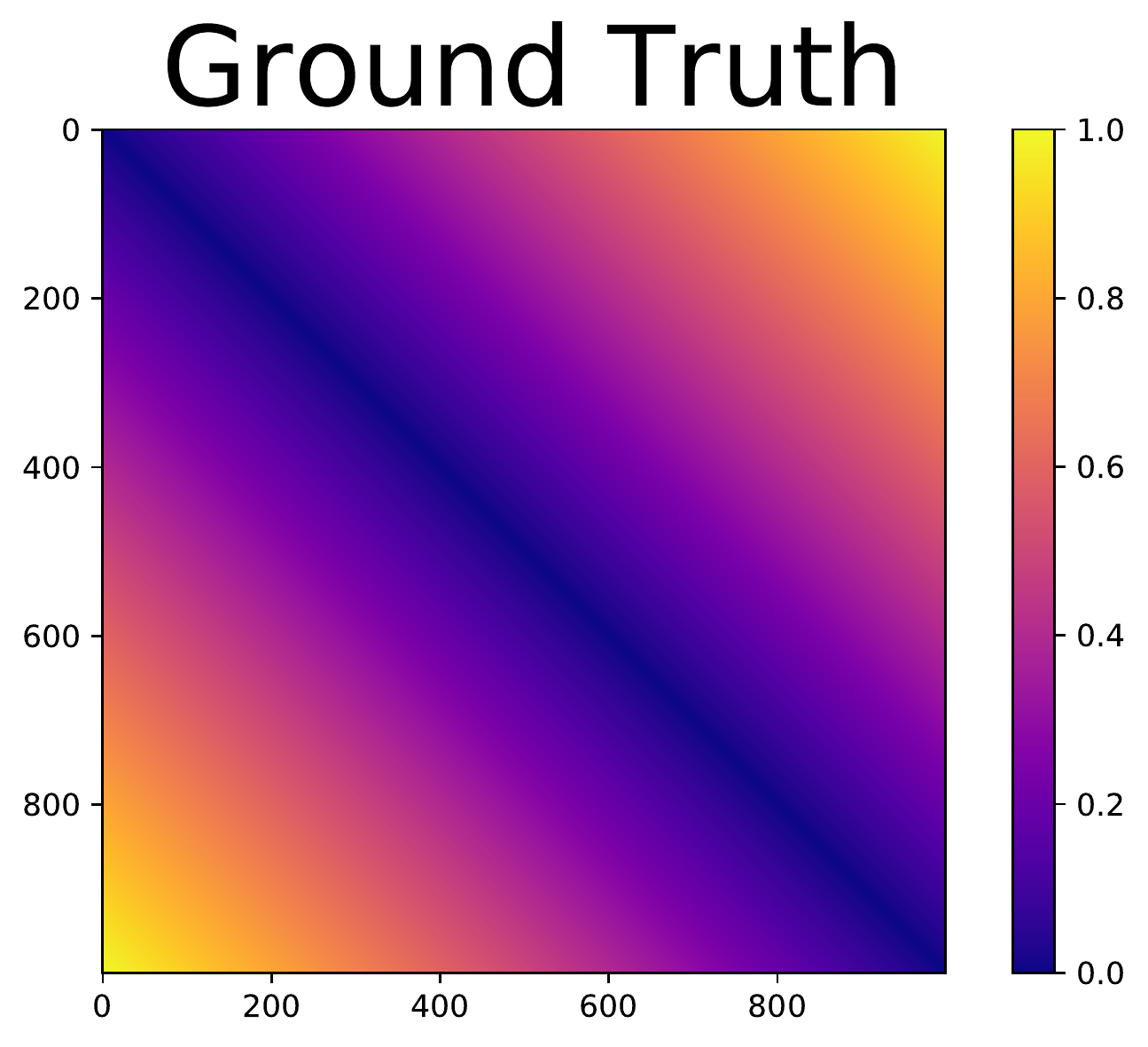}\label{fig:example1b}
    }
    \vspace{-8pt}
    \caption{\small{Illustrations of learning results obtained by various methods for different graphons. 
    In both (a) and (b), we visualize the graphon and its estimations with size $1,000\times 1,000$, and each estimation is derived based on $10$ graphs with less than $300$ nodes. 
    The node degrees of the graphs provide strong evidence to align graphs when learning the graphon in (a) but are useless for the graphon in (b). Our GWB method outperforms state-of-the-art methods. Especially in the challenging case (b), the estimation derived by our method can be aligned to the ground truth by a measure-preserving mapping, which is close to the ground truth under the cut distance.}}
    \label{fig:example1}
\end{figure*}

These methods require the observed graphs to be \textit{well-aligned}\footnote{``Well-aligned'' graphs have comparable size and the correspondence between their nodes is provided.} and generated by a single graphon. 
However, real-world graphs, $e.g.$, the social networks collected from different platforms and different time slots, often have complicated clustering structure, and the correspondence between their nodes is unknown in general. 
This violation limits the feasibility of the above learning methods in practice.
Specifically, these methods have to solve a multi-graph matching problem before learning graphons. 
Because of the NP-hardness of the matching problem, this preprocessing often introduces severe noise to the subsequent learning problem and leads to undesirable learning results.

To overcome the aforementioned challenges, we propose a new method to learn one or multiple graphons from unaligned graphs. 
Our method leverages step functions to estimate graphons.
It minimizes the Gromov-Wasserstein distance (GWD)~\cite{memoli2011gromov} between the step function of each observed graph and that of the target graphon, whose solution is a Gromov-Wasserstein barycenter (GWB) of the graphs~\cite{xu2019scalable}. 
We demonstrate that this learning strategy minimizes an upper bound of the cut distance~\cite{lovasz2012large} between the graphon and its step function, which leads to a computationally-efficient algorithm. 

To the best of our knowledge, our work makes the first attempt to learn graphons from unaligned graphs. 
Different from existing methods, which first match graphs heuristically and then estimate graphons, our method leverages the permutation-invariance of the GW distance and integrates graph matching implicitly in the estimation phase. 
As a result, our method mitigates bias caused by undesired matching processes. 
Given a graphon $W(x,y)$, if its marginal $W(y)=\int_{x\in\Omega}W(x,y)dx$ (or $W(x)=\int_{y\in\Omega}W(x,y)dy$) is very different from a constant function, the graphs generated by it can be aligned readily by sorting and matching their nodes according to their degrees. 
On the contrary, if its marginal is close to a constant function, it will be hard to align its graphs because the node degrees of the graphs' nodes are almost the same. 
As illustrated in Figure~\ref{fig:example1}, no matter whether it is easy to align the graphs or not, our method can successfully learn the graphons and consistently outperforms existing methods. 
Besides the basic GWB method, we design a smoothed GWB method to enhance the continuity of learned graphons. 
Additionally, to learn multiple graphons from the graphs with unknown clustering structures, we propose a mixture model of GWBs. 
These structured GWB models achieve encouraging learning results in some complicated scenarios.

\section{Proposed Method}
A graphon $W:\Omega^2\mapsto [0, 1]$ is defined on a probability space $(\Omega, \mu)$, where $\mu$ is a probability measure on the space $\Omega$. 
Each $W$ formulates a space of graphons, denoted as $\mathcal{W}$. 
Let $\{\mathcal{G}_m\}_{m=1}^M$ be a set of graphs generated by an unknown graphon $W$, whose sampling process is shown in (\ref{eq:generate_graph}). 
We want to estimate the graphon based on the observed graphs, making the estimation close to the ground truth under a specific metric.

\subsection{Approximate graphons by step functions}
A graphon can always be approximated by a step function in the cut norm~\cite{frieze1999quick}. 
For each $W\in\mathcal{W}$, its cut norm is defined as
\begin{eqnarray}\label{eq:cut_n}
\begin{aligned}
\|W\|_{\square}:=\sideset{}{_{\mathcal{X}, \mathcal{Y}\subset \Omega}}\sup\Bigl| \int_{\mathcal{X}\times\mathcal{Y}} W(x, y)dxdy\Bigr|,
\end{aligned}
\end{eqnarray}
where the supremum is taken over all measurable subsets $\mathcal{X}$ and $\mathcal{Y}$ of $\Omega$. 
Based on the cut norm, we can define a commonly-used metric called \textit{cut distance}~\cite{lovasz2012large} between $W_1, W_2\in\mathcal{W}$:
\begin{eqnarray}\label{eq:cut_d}
\begin{aligned}
\delta_{\square}(W_1, W_2) := \sideset{}{_{\phi \in \mathcal{S}_{\Omega}}}\inf\|W_1 - W_2^{\phi}\|_{\square},
\end{aligned}
\end{eqnarray}
where $\mathcal{S}_{\Omega}$ represents the set of measure-preserving mappings from $\Omega$ to $\Omega$.
Accordingly, we have $W_2^{\phi}(x,y)=W_2(\phi(x),\phi(y))$. 
The cut distance plays a central role in graphon theory. 
We say that two graphons $W_1$, $W_2$ are equivalent if $\delta_{\square}(W_1, W_2)=0$, denoted as $W_1\cong W_2$. 
The work in~\cite{borgs2008convergent} demonstrates that the quotient space $\widehat{\mathcal{W}}:=\mathcal{W}\setminus \cong$ is homeomorphic to the set
of graphons and $(\widehat{\mathcal{W}}, \delta_{\square})$ is a compact metric space. 
Similarly, we can define $\delta_1(W_1, W_2):=\sideset{}{_{\phi \in \mathcal{S}_{\Omega}}}\inf\|W_1 - W_2^{\phi}\|_{1}$, where $\|W\|_1:=\int_{\mathcal{X}\times\mathcal{Y}} | W(x, y)|dxdy$. 
According to their definitions, we have
\begin{eqnarray}\label{eq:cmp}
\begin{aligned}
\delta_{\square}(W_1, W_2) \leq \delta_{1}(W_1, W_2),~\forall~ W_1, W_2\in\mathcal{W}.
\end{aligned}
\end{eqnarray}

Let $\mathcal{P}=(\mathcal{P}_1,..,\mathcal{P}_K)$ be a partition of $\Omega$ into $K$ measurable sets. 
We define a step function $W_{\mathcal{P}}: \Omega^2\mapsto [0, 1]$ as
\begin{eqnarray}\label{eq:step}
\begin{aligned}
W_{\mathcal{P}}(x,y)=\sideset{}{_{k,k'=1}^{K}}\sum w_{kk'}1_{\mathcal{P}_k\times \mathcal{P}_{k'}}(x,y),
\end{aligned}
\end{eqnarray}
where each $w_{kk'}\in [0, 1]$ and the indicator function $1_{\mathcal{P}_k\times \mathcal{P}_{k'}}(x,y)$ is 1 if $(x, y)\in\mathcal{P}_{k}\times\mathcal{P}_{k'}$, otherwise it is 0. 
The weak regularity lemma~\cite{lovasz2012large} shown below guarantees that every graphon can be approximated well in the cut norm by step functions.
\begin{theorem}[Weak Regularity Lemma~\cite{lovasz2012large}]\label{thm:wrl}
For every graphon $W\in\mathcal{W}$ and $K\geq 1$, there always exists a step function $W_{\mathcal{P}}$ with $|\mathcal{P}|=K$ steps such that
\begin{eqnarray}\label{eeq:wrl}
\begin{aligned}
\|W - W_{\mathcal{P}}\|_{\square} \leq \frac{2}{\sqrt{\log K}}\|W\|_{L_2}.
\end{aligned}
\end{eqnarray}
\end{theorem}
Note that a corollary of this lemma is $\delta_{\square}(W, W_{\mathcal{P}}) \leq \frac{2}{\sqrt{\log K}}\|W\|_{L_2}$ because $\delta_{\square}(W, W_{\mathcal{P}})\leq \|W - W_{\mathcal{P}}\|_{\square}$.

\subsection{Oracle estimator}
Based on the weak regularity lemma, we would like to learn a step function ${W}_{\mathcal{P}}$ from observed graphs $\{\mathcal{G}_m\}_{m=1}^M$ such that the cut distance between the step function and the ground truth, $i.e.$, $\delta_{\square}(W, {W}_{\mathcal{P}})$, is minimized. 
Note that a graph $\mathcal{G}$ can also be represented as a step function.
\begin{definition}\label{def:step}
For a graph with a node set $\mathcal{V}=\{1,...,N\}$ and an adjacency matrix $\bm{A}$, we can represent it as a step function with $N$ equitable partitions of $\Omega$, $i.e.$, $\mathcal{P}=\{\mathcal{P}_n\}_{n=1}^{N}$,\footnote{The equitable partitions have the same size, $i.e.$, $|\mathcal{P}_n|=|\mathcal{P}_{n'}|$ for all $n\neq n'$.} denoted as $G_{\mathcal{P}}$, where $G_{\mathcal{P}}(x, y)=\frac{1}{N^2}\sum_{n,n'=1}^{N} a_{nn'}1_{\mathcal{P}_{n}\times \mathcal{P}_{n'}}(x,y)$. 
\end{definition}

Ideally, if we know the positions of a graph's nodes, $i.e.$, the $v_n$'s in (\ref{eq:generate_graph}), we can derive an isomorphism of the graph according to the order of the positions and obtain an ``oracle'' step function, denoted as $\hat{\mathcal{G}}$ and $\hat{G}_{\mathcal{P}}$, respectively. 
Applying this sorting operation to $\{\mathcal{G}_m\}_{m=1}^M$, we obtain a set of well-aligned graphs $\{\hat{\mathcal{G}}_m\}_{m=1}^M$ and a set of oracle step functions $\{\hat{G}_{m, \mathcal{P}^m}\}_{m=1}^M$, where the number of partitions $|\mathcal{P}^m|$ is equal to the number of nodes in $\hat{\mathcal{G}}_m$.
Accordingly, we achieve an oracle estimator of $W$ as follows:
\begin{eqnarray}\label{eq:oracle}
\begin{aligned}
W_{O} = \frac{1}{M}\sideset{}{_{m=1}^{M}}\sum \hat{G}_{m, \mathcal{P}^m}.
\end{aligned}
\end{eqnarray}
This oracle estimator provides a consistent estimation of $W$:
\begin{theorem}\label{thm:oracle}
For every $W\in\mathcal{W}$, let $\{\hat{G}_{m, \mathcal{P}^m}\}_{m=1}^M$ be a set of oracle step functions defined by Definition~\ref{def:step}. 
We have
\begin{eqnarray}\label{eq:bound1}
\delta_{\square}(W, W_O)\leq \frac{C}{\min_m |\mathcal{P}^m|},
\end{eqnarray}
where $C$ is a constant.
\end{theorem}
\begin{proof}
For an arbitrary graphon $W$ and its oracle estimator $W_{O}$, we have 
\begin{eqnarray}
\begin{aligned}
\delta_{\square}(W, W_O)&=\delta_{\square}(W, \frac{1}{M}\sideset{}{_{m=1}^{M}}\sum \hat{G}_{m, \mathcal{P}^m})\\
&\leq \Bigl\|W - \frac{1}{M}\sideset{}{_{m=1}^{M}}\sum \hat{G}_{m, \mathcal{P}^m}\Bigr\|_{L_2}\\
&\leq \frac{1}{M}\sideset{}{_{m=1}^{M}}\sum \|W - \hat{G}_{m, \mathcal{P}^m}\|_{L_2}\\
&\leq \frac{1}{M}\sideset{}{_{m=1}^{M}}\sum\frac{C}{|\mathcal{P}^m|}
\leq \frac{C}{\min_m |\mathcal{P}^m|}.
\end{aligned}
\end{eqnarray}
The first inequality is based on the fact that for arbitrary two graphons, their cut distance is smaller than their $L_2$ distance~\cite{janson2013graphons}. 
The second inequality is the triangle inequality.
The third inequality is a corollary of the step function approximation lemma in~\cite{chan2014consistent}, whose derivation corresponds to the second proof shown in the supplementary file of the reference.
In particular, the constant $C$ corresponds to the supremum of the absolute difference between the graphon $W$ and its step function $\hat{G}$, which is independent of the number of partitions. 
\end{proof}

\subsection{Learning graphons via GW barycenters}
The oracle estimator above is unavailable in practice because real-world graphs are generally unaligned -- neither the positions of their nodes nor the correspondence between them is provided. 
Given such unaligned graphs, traditional learning methods first match observed graphs and then estimate the oracle step functions. 
As illustrated in Figure~\ref{fig:example1b}, this strategy often leads to failures because the matching step is NP-hard and creates wrongly-aligned graphs.

To mitigate the dependency on well-aligned graphs (and their oracle step functions), we propose a new learning strategy. 
Specifically, considering the oracle estimator, we have 
\begin{eqnarray}\label{eq:tri}
\begin{aligned}
&\delta_{\square}(W, W_{\mathcal{P}})\\
\leq &\delta_{\square}(W, W_O) + \delta_{\square}(W_O, W_{\mathcal{P}})\\
= &\delta_{\square}(W, W_O) + \delta_{\square}(\frac{1}{M}\sideset{}{_{m=1}^{M}}\sum \hat{G}_{m, \mathcal{P}^m}, W_{\mathcal{P}})\\
\leq &\delta_{\square}(W, W_O) + \frac{1}{M}\sideset{}{_{m=1}^{M}}\sum \delta_{\square}(\hat{G}_{m, \mathcal{P}^m}, W_{\mathcal{P}})\\
=&\delta_{\square}(W, W_O) + \frac{1}{M}\sideset{}{_{m=1}^{M}}\sum \delta_{\square}(G_{m, \mathcal{P}^m}, W_{\mathcal{P}})\\
\leq &\delta_{\square}(W, W_O) + \frac{1}{M}\sideset{}{_{m=1}^{M}}\sum \delta_{1}(G_{m, \mathcal{P}^m}, W_{\mathcal{P}}).
\end{aligned}
\end{eqnarray}
The first inequality in (\ref{eq:tri}) is the triangle inequality, and
the second inequality is derived according to the definition of cut distance.
The latter is manifested because $\delta_{\square}(G_{\mathcal{P}}, \hat{G}_{\mathcal{P}})=0$ ($i.e.$, $G_{\mathcal{P}} \cong \hat{G}_{\mathcal{P}}$).
Finally, the third inequality is based on (\ref{eq:cmp}).
Theorem~\ref{thm:oracle} and (\ref{eq:tri}) indicate that we can minimize an upper bound of $\delta_{\square}(W, W_{\mathcal{P}})$ by solving the following optimization problem:
\begin{eqnarray}\label{eq:opt1}
\begin{aligned}
\sideset{}{_{W_{\mathcal{P}}}}\min \frac{1}{M}\sideset{}{_{m=1}^{M}}\sum \delta_{1}(G_{m, \mathcal{P}^m}, W_{\mathcal{P}}).
\end{aligned}
\end{eqnarray}
This strategy does not need to estimate the oracle step function, because it directly considers the $\delta_1$ distance between observed graphs and the proposed step function. 
To solve this problem, we derive a computationally-efficient alternative of $\delta_1$ based on its equivalent definition shown below. 
\begin{theorem}[Remark 6.13 in~\cite{janson2013graphons}]~\label{thm:eq}
Let $W_1$ and $W_2$ be two graphons defined on the probability spaces
$(\Omega_1, \mu_1)$ and $(\Omega_2, \mu_2)$, respectively. 
The $\delta_{1}(W_1, W_2)$ can be equivalently defined as
$\inf_{\pi\in \Pi(\mu_1, \mu_2)}\int_{(\Omega_1\times\Omega_2)^2}|W_1(x,y)-W_2(x', y')|d\pi(x,x')d\pi(y,y')$, 
where $\Pi(\mu_1,\mu_2)=\{\pi~|~\pi\geq0,~\int_{y\in\Omega_2}d\pi(x,y)=\mu_1,~\int_{x\in\Omega_1}d\pi(x,y)=\mu_2\}$ contains all measures on $\Omega_1\times\Omega_2$ having marginals $\mu_1$ and $\mu_2$.
\end{theorem}
The characterization shown in Theorem~\ref{thm:eq} coincides with the 1-order Gromov-Wasserstein distance between $(\Omega_1, \mu_1)$ and $(\Omega_2, \mu_2)$~\cite{memoli2011gromov}. 
Let $W_{1,\mathcal{P}}$ and $W_{2,\mathcal{Q}}$ be two step functions defined
on $(\Omega_1, \mu_1)$ and $(\Omega_2, \mu_2)$, which have equitable partitions $\mathcal{P}=\{\mathcal{P}_i\}_{i=1}^{I}$ and $\mathcal{Q}=\{\mathcal{Q}_j\}_{j=1}^{J}$. 
We rewrite $W_1(x,y)-W_2(x', y')$ as $\sum_{i,j=1}^{I}w_{1,ij}1_{\mathcal{P}_i\times \mathcal{P}_j}(x,y)-\sum_{i',j'=1}^{J}w_{2,i'j'}1_{\mathcal{Q}_{i'}\times \mathcal{Q}_{j'}}(x', y')$ and denote $r_{iji'j'}$ as $|w_{1,ij}-w_{2,i'j'}|$. 
Let the probability measures $\mu_1$ and $\mu_2$ be constant in each partition, $i.e.$, $\mu_1(x)=\sum_i \mu_{1,i}1_{\mathcal{P}_i}(x)$ and $\mu_2(x)=\sum_i \mu_{2,j}1_{\mathcal{Q}_j}(x)$.
We can then rewrite the $\delta_1$ distance between the two step functions as
\begin{eqnarray*}\label{eq:gwd}
\begin{aligned}
&\delta_{1}(W_{1,\mathcal{P}}, W_{2,\mathcal{Q}})\\
=&\inf_{\pi\in \Pi(\mu_1, \mu_2)}\sum_{i,i',j,j'}\int_{\mathcal{P}_i\times\mathcal{P}_{j}\times\mathcal{Q}_{i'}\times\mathcal{Q}_{j'}}r_{iji'j'}d\pi(x,x')d\pi(y,y')\\
=&\inf_{\pi\in \Pi({\mu}_1, {\mu}_2)}\sum_{i,i',j,j'}r_{iji'j'}\int_{\mathcal{P}_i\times\mathcal{Q}_{i'}}d\pi(x,x')\int_{\mathcal{P}_{j}\times\mathcal{Q}_{j'}}d\pi(y,y')\\
=&\sideset{}{_{\bm{T}\in\Pi(\bm{\mu}_1,\bm{\mu}_2)}}\min\sideset{}{_{i,i',j,j'}}\sum r_{iji'j'}T_{ii'}T_{jj'}
=d_{\text{gw},1}(\bm{W}_1, \bm{W}_2),
\end{aligned}
\end{eqnarray*}
where $\bm{W}_1=[w_{1,ij}]\in [0, 1]^{I\times I}$ and $\bm{W}_2=[w_{2,i'j'}]\in [0, 1]^{J\times J}$ rewrite the step functions in matrix form.
Vectors $\bm{\mu}_1=[\mu_{1,i}]$ and $\bm{\mu}_2=[\mu_{2,j}]$ represents the probability measures $\mu_1$ and $\mu_2$;
$\bm{T}=[T_{ii'}]\in\mathbb{R}^{I\times J}$ is a doubly-stochastic matrix in the set $\Pi(\bm{\mu}_1, \bm{\mu}_2)=\{\bm{T}\geq \bm{0} | \bm{T}\bm{\mu}_{2}=\bm{\mu}_1, \bm{T}^{\top}\bm{\mu}_1=\bm{\mu}_2\}$, whose element $T_{ii'}=\int_{\mathcal{P}_i\times\mathcal{Q}_{i'}}d\pi(x,x')$. 
The optimal $\bm{T}$, denoted as $\bm{T}^*$, is called the optimal transport or optimal coupling between $\bm{\mu}_1$ and $\bm{\mu}_2$~\cite{villani2008optimal,peyre2019computational}.

The derivation above shows that instead of solving a complicated optimization problem in a function space, we can convert it to the 1-order Gromov-Wasserstein distance between matrices~\cite{peyre2016gromov,chowdhury2019gromov,xu2019gromov}. 
Moreover, when we replace the $r_{iji'j'}$ with $r_{iji'j'}^2$, we obtain the squared 2-order Gromov-Wasserstein distance:
\begin{eqnarray}\label{eq:gwd2}
\begin{aligned}
d_{\text{gw},2}^2(\bm{W}_1, \bm{W}_2)=&\sideset{}{}\min_{\bm{T}\in\Pi(\bm{\mu}_1,\bm{\mu}_2)}\sideset{}{_{i,i',j,j'}}\sum r_{iji'j'}^2 T_{ii'}T_{jj'}\\
=&\sideset{}{_{\bm{T}\in\Pi(\bm{\mu}_1,\bm{\mu}_2)}}\min \langle \bm{D} - 2\bm{W}_1\bm{T}\bm{W}_2^{\top}, \bm{T} \rangle.
\end{aligned}
\end{eqnarray}
Here, $\langle \cdot, \cdot\rangle$ calculates the inner product of two matrices. $\bm{D}=(\bm{W}_1\odot \bm{W}_1)\bm{\mu}_1\bm{1}_{J}^{\top} + \bm{1}_{I}\bm{\mu}_2^{\top}(\bm{W}_2\odot \bm{W}_2)$, where $\bm{1}_{I}$ is an $I$-dimensional all-one vector and $\odot$ represents the Hadamard product of matrix. 
Because the 2-order GW distance and the 1-order GW distance are equivalent (pseudo) metrics (Theorem 5.1 in~\cite{memoli2011gromov}), the $d_{\text{gw},2}^2(\bm{W}_1, \bm{W}_2)$ also provides a good alternative for the cut distance of the step functions.
Plugging (\ref{eq:gwd2}) into (\ref{eq:opt1}), the learning problem becomes estimating a GW barycenter of the observed graphs~\cite{peyre2016gromov}:
\begin{eqnarray}\label{eq:opt2}
\begin{aligned}
\sideset{}{_{\bm{W}\in [0, 1]^{K\times K}}}\min \frac{1}{M}\sideset{}{_{m=1}^{M}}\sum d^2_{\text{gw},2}(\bm{A}_m, \bm{W}),
\end{aligned}
\end{eqnarray}
where $\bm{A}_m$ is the adjacency matrix of the graph $\mathcal{G}_m$ and $\bm{W}=[w_{kk'}]\in [0, 1]^{K\times K}$ is the matrix representation of step function $W_{\mathcal{P}}$.
Note that the number of partitions $K$ and the probability measures associated with $\bm{W}$ and $\{\bm{A}_m\}_{m=1}^{N}$ are predefined. 
In the following subsection, we detail how to solve (\ref{eq:opt2}).

\subsection{Implementation details}
\textbf{Setting the number of partitions} 
Given a set of graphs $\{\mathcal{G}_m\}_{m}^{M}$, we denote $N_{\max}$ as the number of the nodes of the largest graph. 
Following the work in~\cite{chan2014consistent,airoldi2013stochastic,channarond2012classification}, we can set the number of partitions to be $K = \lfloor \frac{N_{\max}}{\log N_{\max}} \rfloor$. 
This setting has been proven helpful to achieve a trade-off between accuracy and computational efficiency.

\textbf{Estimating probability measures}
For the observed graphs, we estimate the probability measures by normalized node degrees~\cite{xu2019scalable}. 
We assume that the probability measure of $\bm{W}$ is sorted, $i.e.$, $\bm{\mu}_W=[\mu_{W,1},...,\mu_{W,K}]$ and $\mu_{W,1}\geq ...\geq \mu_{W,K}$, which is estimated by sorting and merging $\{\bm{\mu}_m\}_{m=1}^{M}$. 
Here, $\bm{\mu}_m = \frac{1}{\|\bm{A}_m\bm{1}_{N_m}\|_1}\bm{A}_m\bm{1}_{N_m}$ for $m=1,...,M$, and 
\begin{eqnarray}\label{eq:mu}
\begin{aligned}
\bm{\mu}_{W}=\frac{1}{M}\sideset{}{_{m=1}^{M}}\sum \text{interp1d}_{K}(\text{sort}(\bm{\mu}_m)),
\end{aligned}
\end{eqnarray}
where $\text{sort}(\cdot)$ sorts the elements of the input vector in descending order, and $\text{interp1d}_{K}(\cdot)$ samples $K$ values from the input vector via linear interpolation. 
This strategy has proven beneficial for calculating the GW distance between graphs, which provides useful information when calculating the optimal transport between each $\bm{A}_m$ and the $\bm{W}$~\cite{xu2019scalable}. 

\textbf{Learning optimal transports}
The computation of the $d_{\text{gw},2}^2(\bm{A}_m,\bm{W})$ is a non-convex, non-smooth optimization problem. 
To solve this problem efficiently, we apply the proximal gradient algorithm developed in~\cite{xu2019gromov}. 
This algorithm reformulates the original problem as a series of subproblems and solves them iteratively. 
In each iteration, the subproblem is
\begin{eqnarray}\label{eq:prox}
\min_{\bm{T}\in\Pi(\bm{\mu}_m,\bm{\mu}_W)} \langle \bm{D}_m - 2\bm{A}_m\bm{T}^{(s)}\bm{W}^{\top}, \bm{T} \rangle + \beta \text{KL}(\bm{T} \lVert \bm{T}^{(s)}),
\end{eqnarray}
where $\bm{T}^{(s)}$ is the previous estimation of $\bm{T}$, $\bm{D}_m=(\bm{A}_m\odot \bm{A}_m)\bm{\mu}_m\bm{1}_{K} + \bm{1}_{N_m}\bm{\mu}_W^{\top}(\bm{W}\odot \bm{W})$.
We fix one transport matrix as its previous estimation in the GW term and add a proximal term as the regularizer. 
Here, the proximal term penalizes the KL-divergence between the transport matrix and its previous estimation, which smooths the update of the transport matrix. 
This problem can be solved by the Sinkhorn scaling algorithm~\cite{sinkhorn1967concerning}, whose convergence rate is linear~\cite{altschuler2017near,xie2020fast}.

\textbf{Learning GW barycenters}
Given the optimal transports $\{\bm{T}_m\}_{m=1}^{M}$, the GW barycenter has a closed-form solution~\cite{peyre2016gromov}:
\begin{eqnarray}\label{eq:average}
\bm{W} = \frac{1}{\bm{\mu}_W\bm{\mu}_{W}^{T}}\sideset{}{_{m=1}^{M}}\sum \bm{T}_m^{\top}\bm{A}_m\bm{T}_m.
\end{eqnarray}
The scheme of our algorithm is shown in Algorithm~\ref{alg1}. 
\begin{algorithm}[t]
\small{
	\caption{Learning Graphons via GWB}
	\label{alg1}
	\begin{algorithmic}[1]
	    \STATE \textbf{Input:} Adjacency matrices $\{\bm{A}_m\}_{m=1}^{M}$. The weight of proximal term $\beta$, the number of iterations $L$, the number of inner Sinkhorn iterations $S$.
	    \STATE Initialize $K = \lfloor \frac{N_{\max}}{\log N_{\max}} \rfloor$, and $\bm{W}\sim \text{Uniform}([0,1])$.
	    \STATE Initialize $\{\bm{\mu}_m\}_{m=1}^{M}$ and $\bm{\mu}_W$ via (\ref{eq:mu})
	    \STATE \textbf{For} $l=1,...,L$:
	    \STATE \quad\textbf{For} $m=1,...,M$:\hfill \texttt{// Solve (\ref{eq:prox})}
	    \STATE \quad\quad Initialize $\bm{T}^{(0)}=\bm{\mu}_m\bm{\mu}_W^{\top}$ and $\bm{a}=\bm{\mu}_m$.
	    \STATE \quad\quad\textbf{For} $s=0,...,S-1$:
	    \STATE \quad\quad\quad $\bm{C}=\exp(-\frac{1}{\beta}(\bm{D}_m - 2\bm{A}_m\bm{T}^{(s)}\bm{W}^{\top}))\odot \bm{T}^{(s)}$
	    \STATE \quad\quad\quad $\bm{b}=\frac{\bm{\mu}_{W}}{\bm{C}^{\top}\bm{a}}$, $\bm{a} = \frac{\bm{\mu}_m}{\bm{C}\bm{b}}$, $\bm{T}^{(s+1)}=\text{diag}(\bm{a}){\bm{C}}\text{diag}(\bm{b})$.
	    \STATE \quad\quad $\bm{T}_m=\bm{T}^{(S)}$.
	    \STATE \quad Update $\bm{W}$ via (\ref{eq:average}).
	    \STATE The graphon $W_{\bm{P}}(x,y)=\sum_{k,k'}w_{kk'}1_{\bm{P}_k\times \bm{P}_{k'}}(x,y)$.
	\end{algorithmic}
}
\end{algorithm}

\section{Structured Gromov-Wasserstein Barycenters}
We extend the above algorithm and propose two kinds of structured Gromov-Wasserstein barycenters to apply our learning method to more complicated scenarios.

\textbf{Smoothed GW Barycenters}
As shown in Figure~\ref{fig:example1b}, the estimated graphons achieved by our method can be discontinuous because of the permutation invariance of Gromov-Wasserstein distance. 
To suppress the discontinuity of the results, we impose a smoothness regularization on the GW barycenters and obtain the following problem: 
\begin{eqnarray}\label{eq:sgwb}
\begin{aligned}
\sideset{}{_{\bm{W}\in [0, 1]^{K\times K}}}\min &\frac{1}{M}\sideset{}{_{m=1}^{M}}\sum \langle \bm{D}_m - 2\bm{A}_m\bm{T}_m\bm{W}^{\top}, \bm{T}_m \rangle\\
&+\alpha \|\bm{L}\bm{W}\bm{L}^{\top}\|_F^2,
\end{aligned}
\end{eqnarray}
where $\bm{T}_m$ is current estimation of the $m$-th optimal transport and $\bm{L}\bm{W}\bm{L}^{\top}$ is the matrix representation of the Laplacian filtering of $\bm{W}$.
The first-order optimality condition of this problem also has a closed-form solution. 
In particular, setting the gradient of the objection to zero, we obtain
\begin{eqnarray}\label{eq:1st}
\begin{aligned}
&2\alpha \bm{L}^{\top}\bm{L}\bm{W}\bm{L}^{\top}\bm{L} + \text{diag}(\bm{\mu}_W)\bm{W}\text{diag}(\bm{\mu}_W)\\
=&\frac{1}{M}\sideset{}{_{m=1}^{M}}\sum\bm{T}_m^{\top}\bm{A}_m\bm{T}_m.
\end{aligned}
\end{eqnarray}
Applying singular value decomposition (SVD) to $\bm{L}^{\top}\bm{L}$, $i.e.$, $\bm{L}^{\top}\bm{L}=\bm{U}\bm{\Lambda}\bm{U}^{\top}$, we rewrite the left side of (\ref{eq:1st}) as $\bm{H}\bm{W}\bm{H}^{H}$, where $\bm{H}=\bm{U}(\sqrt{2\alpha}\bm{\Lambda}+i\text{diag}(\bm{\mu}_W))\bm{U}^{\top}$ is a symmetric complex matrix and $\bm{H}^{H}=\bm{U}(\sqrt{2\alpha}\bm{\Lambda}-i\text{diag}(\bm{\mu}_W))\bm{U}^{\top}$ is its Hermitian transpose. 
Therefore, we obtain a smoothed $\bm{W}$ by replacing line 11 of Algorithm~\ref{alg1} with
$\bm{W} = \frac{1}{M}\bm{H}^{-1}\left(\sideset{}{_{m=1}^{M}}\sum\bm{T}_m^{\top}\bm{A}_m\bm{T}_m\right)\bm{H}^{-H}$.

\textbf{Mixed GW Barycenters}
When the observed graphs are generated by $C$ different graphons, we can build a graphon mixture model and learn it as mixed GW barycenters:
\begin{eqnarray}\label{eq:mgwb}
\begin{aligned}
\min_{\{\bm{W}_c\}_{c=1}^{C},\bm{P}\in\Pi(\frac{1}{C}\bm{1}_C,\frac{1}{M}\bm{1}_M)} \sideset{}{}\sum_{c=1}^{C}\sideset{}{}\sum_{m=1}^{M} p_{cm} d^2_{\text{gw},2}(\bm{A}_m, \bm{W}_c).
\end{aligned}
\end{eqnarray}
where we set $\bm{P}=[p_{cm}]$ as a doubly stochastic matrix, whose marginals are $\frac{1}{C}\bm{1}_C$ and $\frac{1}{M}\bm{1}_M$. 
The value $p_{cm}$ indicates the joint probability that the $m$-th graph is generated by the $c$-th graphon. 
In other words, the objective function of (\ref{eq:mgwb}) leads to a hierarchical optimal transport problem~\cite{luo2020hierarchical}, in which the ground distance is defined by the GW distance and $\bm{P}$ is the optimal transport matrix.
This problem is solved by alternating optimization.
In particular, replacing $\frac{1}{M}$ with $p_{cm}$, we still apply Algorithm~\ref{alg1} to learn $\{\bm{W}_c\}_{c=1}^{C}$.
Given $\{\bm{W}_c\}_{c=1}^{C}$, we calculate the ground distance matrix $\bm{D}_{\text{gw}}=[d^2_{\text{gw},2}(\bm{A}_m, \bm{W}_c)]$ and update $\bm{P}$ by solving an optimal transport problem with an entropy regularization.
\begin{eqnarray}
\begin{aligned}
\sideset{}{_{\bm{P}\in\Pi(\frac{1}{C}\bm{1}_C,\frac{1}{M}\bm{1}_M)}}\min \langle\bm{D}_{\text{gw}},\bm{P}\rangle + \beta \langle\bm{P},\log\bm{P}\rangle.
\end{aligned}
\end{eqnarray}
Similar to (\ref{eq:prox}), this problem can also be solved by the Sinkhorn scaling algorithm in~\cite{sinkhorn1967concerning}.

\section{Related Work}
\textbf{Graphon estimation} 
As a classic graphon estimation method, the stochastic block approximation (SBA) learns stochastic block models as graphons~\cite{airoldi2013stochastic}.
The block size of the method can be optimized heuristically by the ``largest gap'' algorithm~\cite{channarond2012classification}. 
The smoothing-and-sorting (SAS) method improves this strategy by adding total-variation denoising as a post-processing step~\cite{chan2014consistent}. 
The work in~\cite{pensky2019dynamic} further extends this strategy and proposes a dynamic stochastic block model to describe time-varying graphons. 
The matrix completion (MC) method~\cite{keshavan2010matrix}, the universal singular value thresholding (USVT) algorithm~\cite{chatterjee2015matrix}, and the spectral method~\cite{xu2018rates} learn low-rank matrices as the proposed step functions. 
The work in~\cite{ruiz2020graphon} represents graphons by their Fourier transformations.
The methods above require that the observed graphs be well-aligned and generated by a single graphon. 
Our work makes the first attempt to learn one or multiple graphons from unaligned graphs.

\begin{table*}[t]
\centering
\begin{small}
    \caption{Comparisons on estimation errors (MSE for ``Easy to align'', $d_{\text{gw},2}$ for ``Hard to align'')}\label{tab:error}
    \vspace{-8pt}
    \begin{tabular}{
    @{\hspace{1pt}}c@{\hspace{2pt}}|
    l@{\hspace{2pt}}|
    @{\hspace{2pt}}c@{\hspace{2pt}}|
    @{\hspace{2pt}}c@{\hspace{2pt}}
    @{\hspace{2pt}}c@{\hspace{2pt}}
    @{\hspace{2pt}}c@{\hspace{2pt}}
    @{\hspace{2pt}}c@{\hspace{2pt}}
    @{\hspace{2pt}}c@{\hspace{2pt}}|
    @{\hspace{2pt}}c@{\hspace{2pt}}
    @{\hspace{2pt}}c@{\hspace{1pt}}}
    \hline\hline
    Type 
    &$W(x,y)$, $x,y\in [0,1]$ 
    &\# nodes
    &SBA &LG &MC &USVT &SAS &GWB &SGWB\\
    \hline
    &\multirow{2}{*}{$xy$} &$200$
    &65.6$\pm$6.5 
    &29.8$\pm$5.7 
    &\textbf{11.3$\pm$0.8} 
    &31.7$\pm$2.5 
    &125.0$\pm$1.3 
    &40.6$\pm$5.7 
    &39.3$\pm$5.5\\
    &&100$\sim$300
    &157.4$\pm$23.3 
    &133.2$\pm$22.8 
    &138.3$\pm$21.3 
    &131.2$\pm$24.2 
    &161.9$\pm$19.5 
    &52.5$\pm$13.1 
    &\textbf{51.9$\pm$12.6}\\
    \cline{2-10}
    &\multirow{2}{*}{$e^{-(x^{0.7}+y^{0.7})}$} &$200$
    &58.7$\pm$7.8 
    &22.9$\pm$3.1 
    &71.7$\pm$0.5 
    &\textbf{12.2$\pm$1.5} 
    &77.7$\pm$0.8 
    &21.6$\pm$2.1 
    &20.9$\pm$1.8\\
    &&100$\sim$300
    &165.2$\pm$22.8 
    &157.6$\pm$24.2 
    &166.2$\pm$21.5 
    &158.2$\pm$24.5 
    &153.4$\pm$25.1 
    &48.6$\pm$11.9 
    &\textbf{48.0$\pm$10.8}\\
    \cline{2-10}
    &\multirow{2}{*}{$\frac{x^2+y^2+\sqrt{x}+\sqrt{y}}{4}$} &$200$
    &63.4$\pm$7.6 
    &24.1$\pm$2.5 
    &73.2$\pm$0.7 
    &33.8$\pm$1.1 
    &99.3$\pm$1.2 
    &18.9$\pm$3.5 
    &\textbf{18.4$\pm$2.8}\\
    &&100$\sim$300
    &258.5$\pm$36.0 
    &254.2$\pm$36.6 
    &259.5$\pm$35.0 
    &254.2$\pm$36.8 
    &240.6$\pm$39.2 
    &81.0$\pm$18.8 
    &\textbf{80.5$\pm$17.8}\\
    \cline{2-10}
    Easy&\multirow{2}{*}{$\frac{1}{2}(x+y)$} &$200$
    &66.2$\pm$8.3 
    &24.0$\pm$2.5 
    &71.9$\pm$0.6 
    &40.2$\pm$0.8 
    &108.3$\pm$1.0 
    &21.2$\pm$4.6 
    &\textbf{20.2$\pm$3.9}\\
    &&100$\sim$300
    &247.6$\pm$40.3 
    &241.3$\pm$41.0 
    &247.0$\pm$39.1 
    &241.3$\pm$41.0 
    &231.3$\pm$40.2 
    &\textbf{83.8$\pm$22.5}
    &84.6$\pm$22.0\\
    \cline{2-10}
    to&\multirow{2}{*}{$\frac{1}{1+\exp(-10(x^2+y^2))}$} &$200$
    &55.0$\pm$9.5 
    &23.1$\pm$3.2 
    &64.6$\pm$0.5 
    &37.3$\pm$0.6 
    &73.3$\pm$0.7 
    &\textbf{14.8$\pm$2.3} 
    &16.1$\pm$1.5\\
    &&100$\sim$300
    &394.0$\pm$45.7 
    &397.0$\pm$46.5 
    &400.7$\pm$45.6 
    &399.3$\pm$47.0 
    &345.4$\pm$52.6 
    &62.8$\pm$12.6 
    &\textbf{62.5$\pm$12.3}\\
    \cline{2-10}
    Align&\multirow{2}{*}{$\frac{1}{1+\exp(-(\max\{x,y\}^2}$} &$200$
    &48.3$\pm$6.1 
    &24.5$\pm$2.3 
    &71.1$\pm$0.4 
    &24.4$\pm$0.5 
    &54.4$\pm$0.5 
    &\textbf{15.3$\pm$1.0} 
    &17.1$\pm$1.4\\
    &&100$\sim$300
    &382.9$\pm$54.7 
    &387.9$\pm$53.6 
    &392.5$\pm$52.3 
    &391.9$\pm$54.8 
    &336.7$\pm$58.6 
    &\textbf{39.8$\pm$8.6} 
    &41.7$\pm$8.3\\
    \cline{2-10}
    (MSE)&\multirow{2}{*}{$e^{-\max\{x,y\}^{3/4}}$} &$200$
    &56.3$\pm$7.1 
    &26.8$\pm$0.9 
    &79.3$\pm$0.5 
    &50.6$\pm$0.3 
    &68.6$\pm$0.6 
    &21.4$\pm$1.7 
    &\textbf{21.2$\pm$1.0}\\
    &&100$\sim$300
    &234.7$\pm$32.9 
    &234.0$\pm$33.3 
    &241.4$\pm$31.2 
    &242.9$\pm$33.3 
    &212.0$\pm$36.5 
    &49.3$\pm$9.5 
    &\textbf{48.7$\pm$9.1}\\
    \cline{2-10}
    &\multirow{2}{*}{$e^{-\frac{\min\{x,y\}+\sqrt{x}+\sqrt{y}}{2}}$} &$200$ 
    &55.7$\pm$7.7 
    &26.4$\pm$5.6 
    &76.4$\pm$0.4 
    &28.3$\pm$0.5 
    &76.4$\pm$0.8 
    &\textbf{23.2$\pm$1.3}
    &23.3$\pm$1.4\\
    &&100$\sim$300
    &232.1$\pm$30.8 
    &231.4$\pm$31.8 
    &238.7$\pm$29.9 
    &232.6$\pm$31.9 
    &208.3$\pm$34.8 
    &48.2$\pm$11.7 
    &\textbf{47.9$\pm$11.0}\\
    \cline{2-10}
    &\multirow{2}{*}{$\log(1+\max\{x,y\})$}  &$200$
    &66.0$\pm$8.4 
    &37.1$\pm$6.6 
    &66.9$\pm$0.8 
    &120.9$\pm$0.5 
    &137.0$\pm$1.2 
    &23.7$\pm$2.5 
    &\textbf{23.2$\pm$1.8}\\
    &&100$\sim$300
    &370.8$\pm$38.9 
    &370.7$\pm$40.4 
    &374.5$\pm$39.5 
    &375.5$\pm$37.5 
    &337.7$\pm$42.1 
    &\textbf{104.0$\pm$18.8} 
    &107.4$\pm$18.9\\
    \hline\hline
    \multirow{2}{*}{Hard}&\multirow{2}{*}{$|x-y|$} &$200$
    &0.202$\pm$0.001 
    &0.200$\pm$0.002 
    &0.206$\pm$0.001 
    &0.215$\pm$0.002 
    &0.217$\pm$0.002 
    &0.057$\pm$0.005 
    &\textbf{0.050$\pm$0.003}\\
    &&100$\sim$300
    &0.254$\pm$0.007 
    &0.254$\pm$0.008 
    &0.254$\pm$0.009 
    &0.261$\pm$0.009 
    &0.248$\pm$0.008 
    &0.085$\pm$0.012 
    &\textbf{0.080$\pm$0.010}\\
    \cline{2-10}
    \multirow{2}{*}{to}&\multirow{2}{*}{$1 - |x-y|$} &$200$
    &0.200$\pm$0.001 
    &0.198$\pm$0.002 
    &0.202$\pm$0.002 
    &0.209$\pm$0.003 
    &0.217$\pm$0.001 
    &0.063$\pm$0.003 
    &\textbf{0.057$\pm$0.001}\\
    &&100$\sim$300
    &0.383$\pm$0.041 
    &0.384$\pm$0.040 
    &0.383$\pm$0.040 
    &0.393$\pm$0.041 
    &0.350$\pm$0.044 
    &\textbf{0.075$\pm$0.009} 
    &0.077$\pm$0.004\\
    \cline{2-10}
    \multirow{2}{*}{Align}&\multirow{2}{*}{$0.8\bm{I}_2\otimes 1_{[0, \frac{1}{2}]^2}$} &$200$
    &0.252$\pm$0.018 
    &0.258$\pm$0.018 
    &0.258$\pm$0.016 
    &0.252$\pm$0.016 
    &0.367$\pm$0.004 
    &0.252$\pm$0.002 
    &\textbf{0.218$\pm$0.001}\\
    &&100$\sim$300
    &0.355$\pm$0.005 
    &0.359$\pm$0.002 
    &0.361$\pm$0.005 
    &0.392$\pm$0.038 
    &0.409$\pm$0.004 
    &\textbf{0.328$\pm$0.034} 
    &0.329$\pm$0.032\\
    \cline{2-10}
    \multirow{2}{*}{($d_{\text{gw},2}$)}&\multirow{2}{*}{$0.8\text{flip}(\bm{I}_2)\otimes 1_{[0, \frac{1}{2}]^2}$}  &$200$
    &0.241$\pm$0.010 
    &0.254$\pm$0.005 
    &0.250$\pm$0.003 
    &0.242$\pm$0.003 
    &0.364$\pm$0.001 
    &0.252$\pm$0.002 
    &\textbf{0.190$\pm$0.058}\\
    &&100$\sim$300
    &0.453$\pm$0.004 
    &0.487$\pm$0.002 
    &0.450$\pm$0.008 
    &0.477$\pm$0.001 
    &0.468$\pm$0.001 
    &0.427$\pm$0.027 
    &\textbf{0.420$\pm$0.027}\\
    \hline\hline
    \end{tabular}
\end{small}
\end{table*}

\textbf{Gromov-Wasserstein distance}
The GW distance has been widely used to measure the difference between structured data, $e.g.$, geometry shapes~\cite{memoli2011gromov} and graphs~\cite{vayer2018fused}. 
For graphs, the optimal transport associated with their GW distance indicates the correspondence between their nodes, which is beneficial for graph matching~\cite{xu2019gromov}. 
To calculate this distance, the work in~\cite{peyre2016gromov} adds an entropy regularizer to the objective function and applies the Sinkhorn scaling algorithm~\cite{cuturi2013sinkhorn}. 
The work in~\cite{xu2019gromov} improves this method by replacing the entropy regularizer with a Bregman proximal term. 
An ADMM-based method is proposed in~\cite{xu2020gwf} to calculate the GW distance between directed graphs. 
Recently, the recursive GW distance~\cite{xu2019scalable} and the sliced GW distance~\cite{titouan2019sliced} have been proposed to reduce the computational complexity of the GW distance. 
A GW barycenter model is proposed in~\cite{peyre2016gromov}, which shows the potential of graph clustering~\cite{xu2020gwf}.
Our work is pioneering in the development of structured GW barycenters to learn graphons.

\section{Experiments}
\subsection{Synthetic data}
To demonstrate the efficacy of our \textbf{GWB} method and its smoothed variant (\textbf{SGWB}), we compare them with existing methods on learning synthetic graphons. 
We set the hyperparameters of our methods as follows: the weight of the proximal term $\beta=0.005$, the number of iterations $L=5$, and the number of Sinkhorn iterations $S=10$; for the SGWB method, the weight of the smoothness regularizer $\alpha=0.0002$. 
The baselines include the stochastic block approximation (\textbf{SBA})~\cite{airoldi2013stochastic}, the ``largest gap'' (\textbf{LG}) based block approximation~\cite{channarond2012classification}, the matrix completion (\textbf{MC})~\cite{keshavan2010matrix}, the universal singular value thresholding (\textbf{USVT})~\cite{chatterjee2015matrix}, and the sorting-and-smoothing (\textbf{SAS})~\cite{chan2014consistent}. 

We prepare 13 kinds of synthetic graphons, whose definitions are shown in Table~\ref{tab:error}. 
The resolution of these graphons is $1000\times 1000$.
Among these graphons, nine are considered in~\cite{chan2014consistent}. 
The graphs generated by them are easy to align -- the node degrees of these graphs provide strong evidence to sort and match nodes. 
For these graphons, we apply the mean-square-error (MSE) to evaluate different methods.
Additionally, to highlight the advantage of our method, we design four challenging graphons, whose graphs are hard to align.\footnote{$0.8\bm{I}_2\otimes 1_{[0, \frac{1}{2}]^2}$ is a graphon with two diagonal blocks, and $0.8\text{flip}(\bm{I}_2)\otimes 1_{[0, \frac{1}{2}]^2}$ is a bipartite graphon, where $\otimes$ represents the Kronecker product and $\bm{I}_2$ is a $2\times 2$ identity matrix.} 
In particular, for the graphs generated by these four graphons, the node degrees of different nodes can be equal to each other.
Therefore, there is no simple way of aligning the graphs.
For these graphons, we apply the GW distance $d_{\text{gw},2}$ as the evaluation measurement. 
For each graphon, we simulate graphs with two different settings: 
In one setting, each of the graphs has $200$ nodes, while in the other setting, the number of each graph's nodes is sampled uniformly from the range $[100, 300]$. 
Originally, the baselines above are designed for the former setting. 
When dealing with graphs with different sizes, they pad zeros to the corresponding adjacency matrices and enforce the graphs to have the same size.
For each setting, we test our method and the baselines in $10$ trials, and  
in each trial we simulate $10$ graphs and estimate the graphon by different methods.

\begin{table*}[t]
\centering
\begin{small}
    \caption{Comparisons on averaged MSE and runtime (second) under different configurations}\label{tab:consistent}
    \vspace{-8pt}
    \begin{tabular}{
    @{\hspace{4pt}}c@{\hspace{4pt}}|
    @{\hspace{4pt}}c@{\hspace{4pt}}|
    @{\hspace{4pt}}c@{\hspace{4pt}}|
    @{\hspace{4pt}}c@{\hspace{4pt}}
    @{\hspace{4pt}}c@{\hspace{4pt}}
    @{\hspace{4pt}}c@{\hspace{4pt}}
    @{\hspace{4pt}}c@{\hspace{4pt}}
    @{\hspace{4pt}}c@{\hspace{4pt}}
    @{\hspace{4pt}}c@{\hspace{4pt}}
    @{\hspace{4pt}}c@{\hspace{4pt}}
    }
    \hline\hline
    Measurement &\# graphs $M$ & \# nodes $N$
    &SBA
    &LG
    &MC
    &USVT
    &SAS
    &GWB
    &SGWB\\
    \hline
    \multirow{6}{*}{MSE}
    &2 &200 
    &62.1$\pm$9.0 &26.5$\pm$3.9 &73.6$\pm$2.3 &45.6$\pm$1.6 &94.1$\pm$1.5 &\textbf{24.6$\pm$4.3} &\textbf{24.3$\pm$3.3}\\
    &10 &200 
    &59.5$\pm$7.7 &26.5$\pm$3.6 &70.7$\pm$0.6 &39.9$\pm$0.9 &91.1$\pm$0.9 &\textbf{22.3$\pm$2.8} &\textbf{21.9$\pm$2.4}\\
    &20 &200 
    &32.2$\pm$2.0 &23.7$\pm$4.1 &50.7$\pm$0.4 &38.7$\pm$0.5 &91.0$\pm$0.7 &\textbf{21.1$\pm$2.0} &\textbf{20.7$\pm$1.7}\\
    &2 &500 
    &59.4$\pm$0.4 &49.1$\pm$2.0 &69.0$\pm$1.8 &35.4$\pm$2.4 &23.9$\pm$4.3 &\textbf{22.2$\pm$7.0} &\textbf{21.1$\pm$4.5}\\
    &10 &500 
    &47.1$\pm$5.6 &20.2$\pm$1.4 &60.6$\pm$0.3 &27.6$\pm$0.4 &34.8$\pm$0.8 &\textbf{19.7$\pm$2.5} &\textbf{19.6$\pm$1.7}\\
    &20 &500 
    &34.1$\pm$3.3 &18.8$\pm$1.7 &43.2$\pm$0.2 &27.0$\pm$0.3 &35.4$\pm$0.7 &\textbf{17.3$\pm$1.9} &\textbf{17.7$\pm$1.4}\\
    \hline
    \multirow{2}{*}{Runtime}
    &10 &200
    &0.69$\pm$0.03 &0.61$\pm$0.04 &0.02$\pm$0.01 &0.02$\pm$0.00 &0.03$\pm$0.01 &0.32$\pm$0.04 &0.34$\pm$0.03\\
    &10 &500
    &3.74$\pm$0.09 &3.69$\pm$0.08 &0.08$\pm$0.02 &0.09$\pm$0.01 &0.08$\pm$0.01 &0.67$\pm$0.06 &0.70$\pm$0.06\\
    \hline\hline
    \end{tabular}
\end{small}
\end{table*}

The experimental results in Table~\ref{tab:error} show that our GWB method outperforms the baselines in most situations. 
Especially for the graphs that are hard to align, the gap between our methods and the baselines become bigger.
When the observed graphs have different sizes, the estimation errors of the baselines increase because the padding step does harm to the alignment of the graphs.
By contrast, our GWB method is relatively robust to the variance of graph size, and  
it achieves much lower estimation errors. 
Additionally, with the help of the smoothness regularizer, our SGWB method improves the stability of the original GWB method, which achieves comparable estimator errors but with smaller standard deviations.
Typical visualization results are shown in Figure~\ref{fig:example1}.

Moreover, our methods are robust across datasets.
In particular, for the nine graphons that are easy to align, we generate different numbers of graphs with different sizes. 
Under each configuration, we record the averaged MSEs of the learning methods in Table~\ref{tab:consistent}, which verifies the robustness of our method. 

Besides estimation errors, we also compare various methods on their computational complexity.
Suppose that we have $M$ graphs, and each of them has $N$ nodes and $E$ edges. 
Learning a step function with $K$ partitions as a graphon, we list the complexity of the methods in Table~\ref{tab:complexity}.
In particular, line 8 of Algorithm~\ref{alg1} involves sparse matrix multiplication, whose complexity is $\mathcal{O}(EK + NK^2)$. 
Because the graphs often have sparse edges, $i.e.$, $E=\mathcal{O}(N\log N)$ and $K=\mathcal{O}(\frac{N}{\log N})$, the complexity of our GWB method is comparable to others when the numbers of iterations ($i.e.$, $L$ and $S$) are small. 
The runtime in Table~\ref{tab:consistent} shows that our GWB and SGWB are generally faster than SBA and LG in practice. 

\begin{table}[t]
\centering
\begin{small}
    \caption{Comparisons on computational complexity}\label{tab:complexity}
    \vspace{-8pt}
    \begin{tabular}{
    @{\hspace{4pt}}c@{\hspace{4pt}}
    @{\hspace{4pt}}l@{\hspace{4pt}}|
    @{\hspace{4pt}}c@{\hspace{4pt}}
    @{\hspace{4pt}}l@{\hspace{4pt}}}
    \hline\hline
    Method & Complexity & Method & Complexity \\
    \hline
    LG   &$\mathcal{O}(MN^2)$ &SBA  &$\mathcal{O}(MKN\log N)$              \\
    MC   &$\mathcal{O}(N^3)$  &SAS  &$\mathcal{O}(MN\log N + K^2\log K^2)$ \\
    USVT &$\mathcal{O}(N^3)$  &GWB  &$\mathcal{O}(LSM(EK + NK^2))$         \\
    \hline\hline
    \end{tabular}
\end{small}
\end{table}

\subsection{Real-world data}
For real-world graph datasets, our mixed GWB method (\textbf{MixGWBs}) provides a new way to cluster graphs.
In particular, by learning $C$ graphons from $M$ graphs, we achieve $C$ centroids of different clusters and the learned optimal transport $\bm{P}=[p_{cm}]$ indicates the probability that the $m$-th graph belongs to the $c$-th cluster.
To demonstrate the effectiveness of our method, we test it on two datasets and compare it with three baselines. 
The datasets are the IMDB-BINARY and the IMDB-MULTI~\cite{yanardag2015deep}, which can be downloaded from~\cite{morris2020tudataset}.
The IMDB-BINARY contains 1000 graphs belonging to two clusters, while 
the IMDB-MULTI contains 1500 graphs belonging to three clusters.
These two datasets are challenging for graph clustering, as the nodes and edges of their graphs do not have any side information. 
We have to cluster graphs merely based on their binary adjacency matrices. 

For these two datasets, the clustering methods based on the GW distance achieve state-of-the-art performance. 
The representative methods include ($i$) the fused Gromov-Wasserstein kernel method (\textbf{FGWK})~\cite{titouan2019optimal}; ($ii$) the K-means using the Gromov-Wasserstein distance as the metric (\textbf{GW-KM}); and ($iii$) the Gromov-Wasserstein factorization (\textbf{GWF}) method~\cite{xu2020gwf}.
We test our MixGWBs method and compare it with these three methods on clustering accuracy. 
In particular, for each dataset, we apply 10-fold cross-validation to evaluate each clustering method.
The averaged clustering accuracy and the standard deviation are shown in Table~\ref{tab:cmp}. 
The performance of our method is at least comparable to that of the competitors. 
Figure~\ref{fig:imdb} visualizes the graphons learned by our method and illustrates the difference between different clusters.
We can find that the graphons correspond to the block models with different block sizes. 
Additionally, the IMDB-MULTI is much more challenging than the IMDB-BINARY, because it contains three rather than two clusters and the block structure of each cluster is not so clear as the clusters of the IMDB-BINARY. 

\begin{table}[t]
	\centering
	\caption{Comparisons on clustering accuracy (\%)}\label{tab:cmp}
	\vspace{-8pt}
	\begin{small}
	\begin{tabular}{
	@{\hspace{4pt}}l@{\hspace{4pt}}|
	@{\hspace{4pt}}c@{\hspace{4pt}}
	@{\hspace{4pt}}c@{\hspace{4pt}}
	@{\hspace{4pt}}c@{\hspace{4pt}}
	@{\hspace{4pt}}c@{\hspace{4pt}}}
	\hline\hline
	    Dataset 
	    &FGWK
	    &GW-KM 
	    &GWF
	    &MixGWBs\\
	    \hline
	    IMDB-BINARY
	    &56.7$\pm$1.5
	    &53.5$\pm$2.3
	    &60.6$\pm$1.7
	    &\textbf{61.4$\pm$1.8}\\
	    IMDB-MULTI
	    &42.3$\pm$2.4
	    &41.0$\pm$3.1
	    &40.8$\pm$2.0
	    &\textbf{42.9$\pm$1.9}\\
	\hline\hline
	\end{tabular}
    \end{small}
\end{table}


\begin{figure}[t!]
    \centering
    \subfigure[IMDB-BINARY]{
    \includegraphics[height=1.5cm]{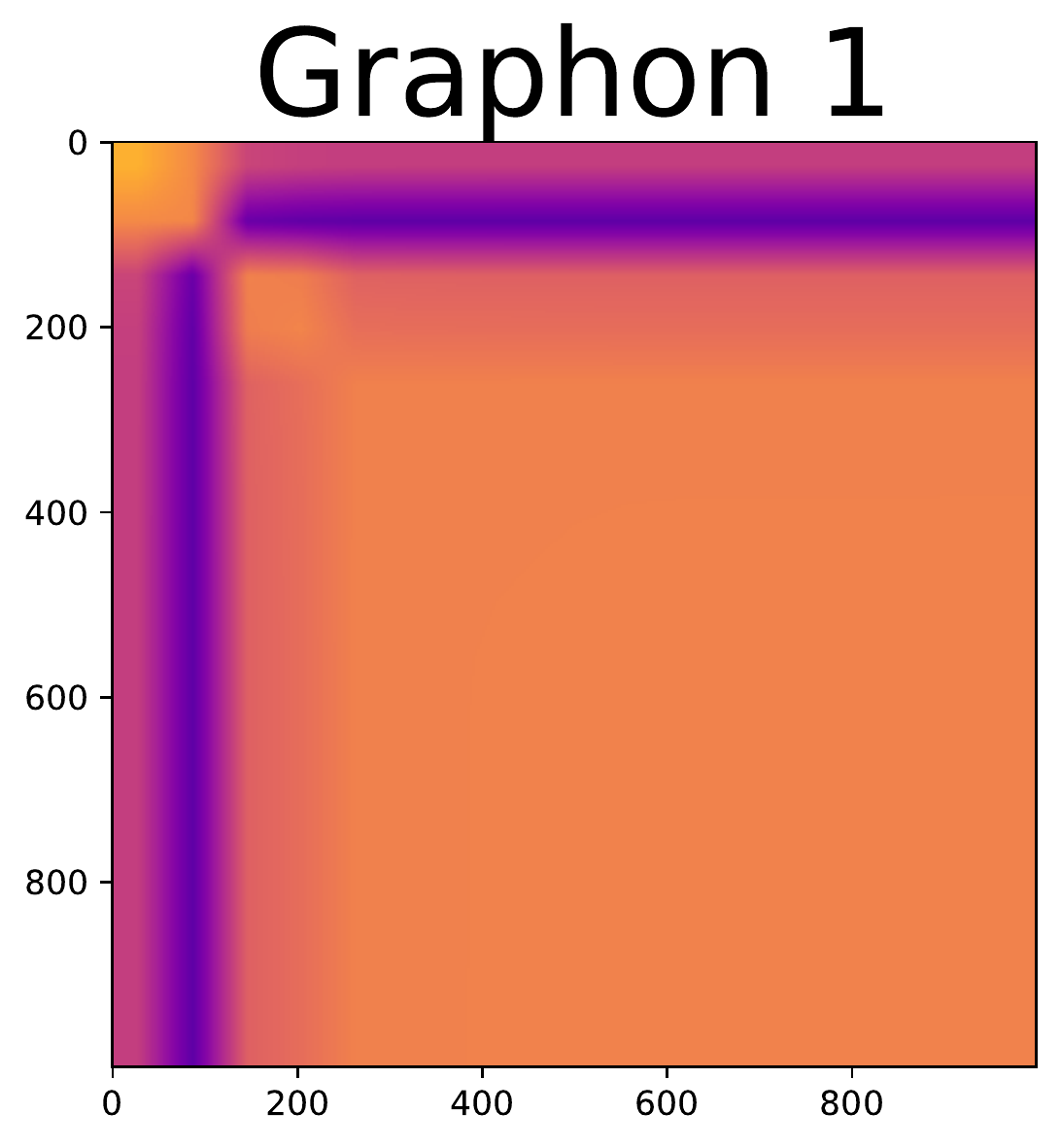}
    \includegraphics[height=1.5cm]{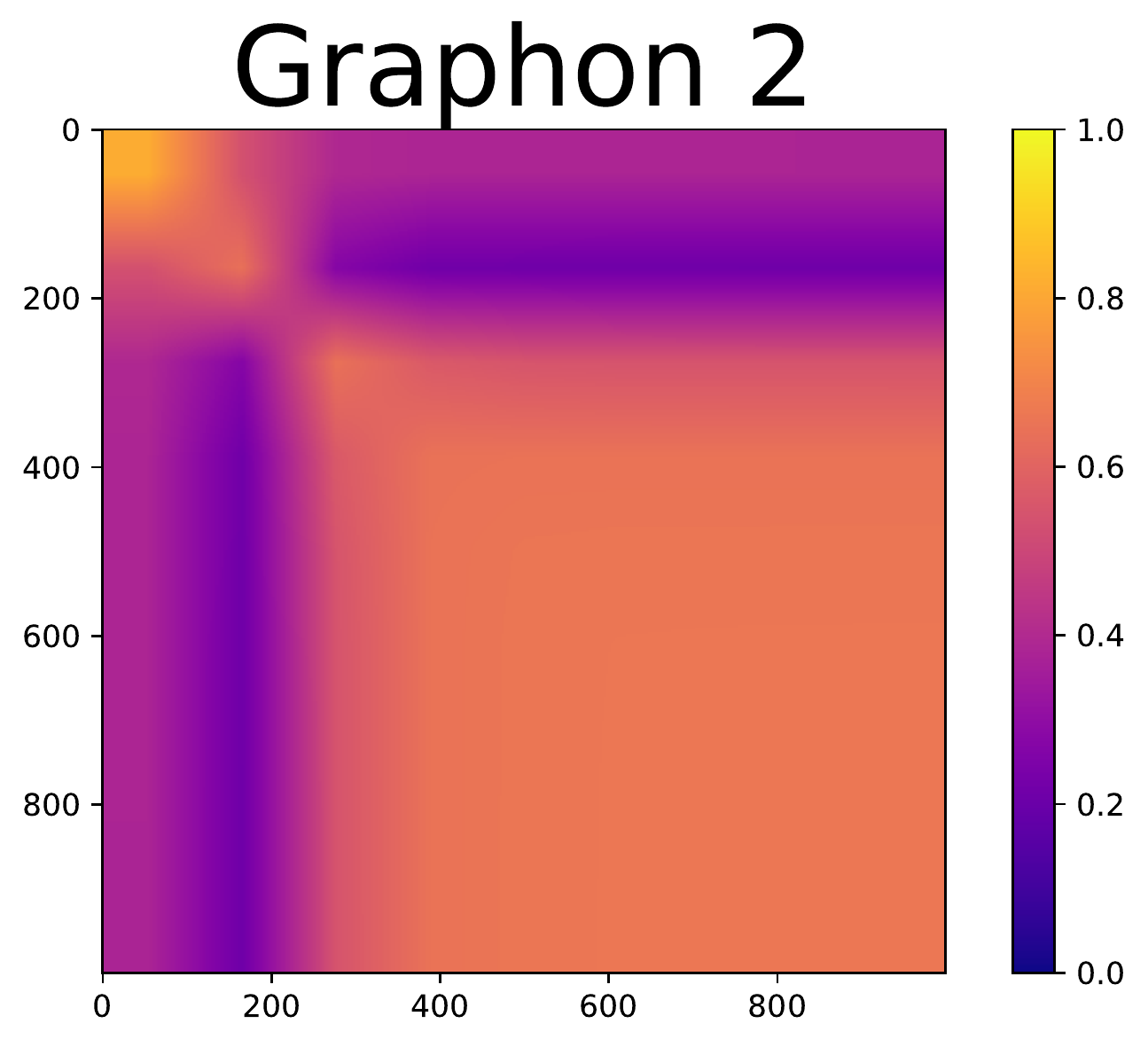}\label{fig:imdb1}
    }
    \subfigure[IMDB-MULTI]{
    \includegraphics[height=1.5cm]{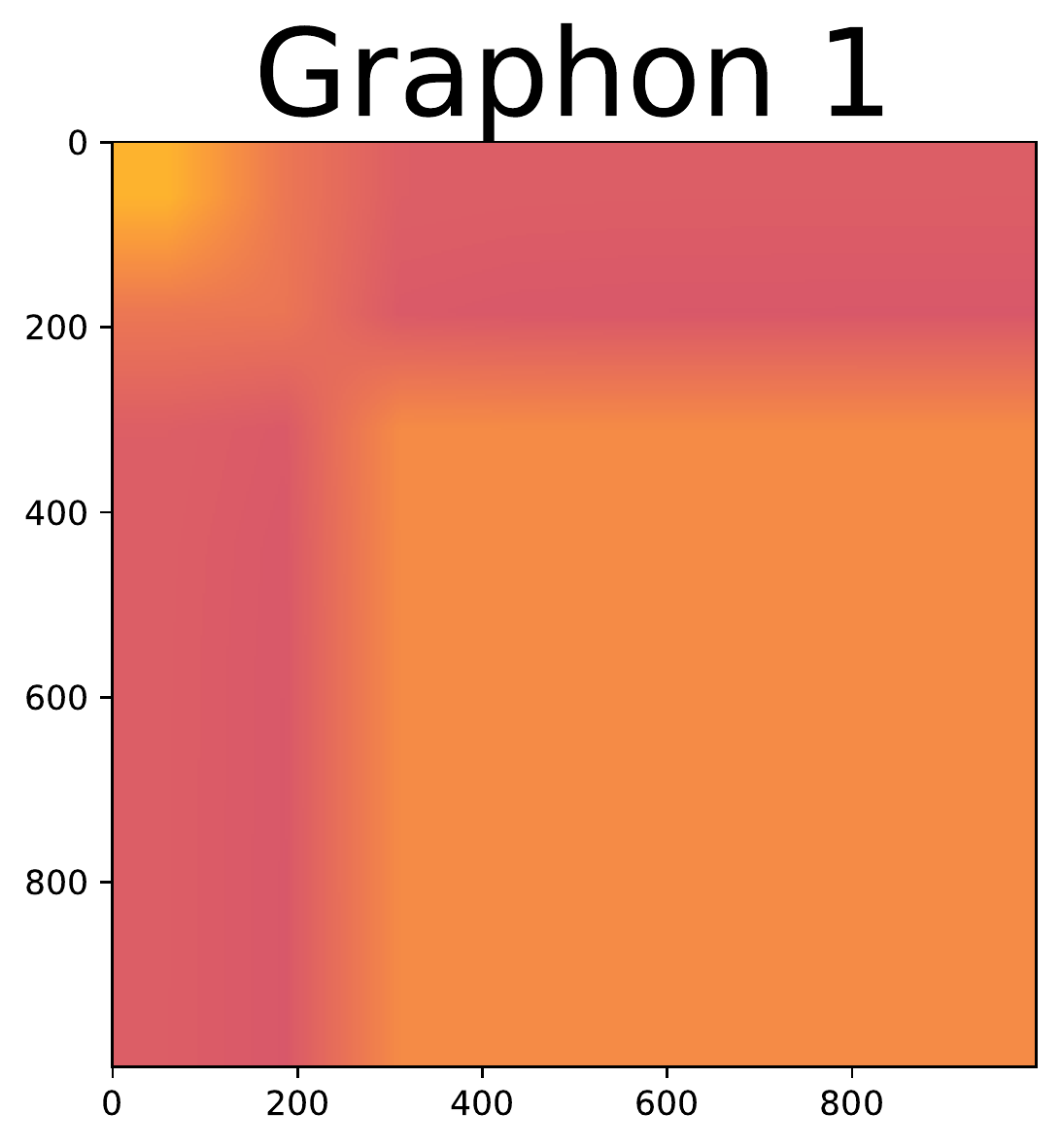}
    \includegraphics[height=1.5cm]{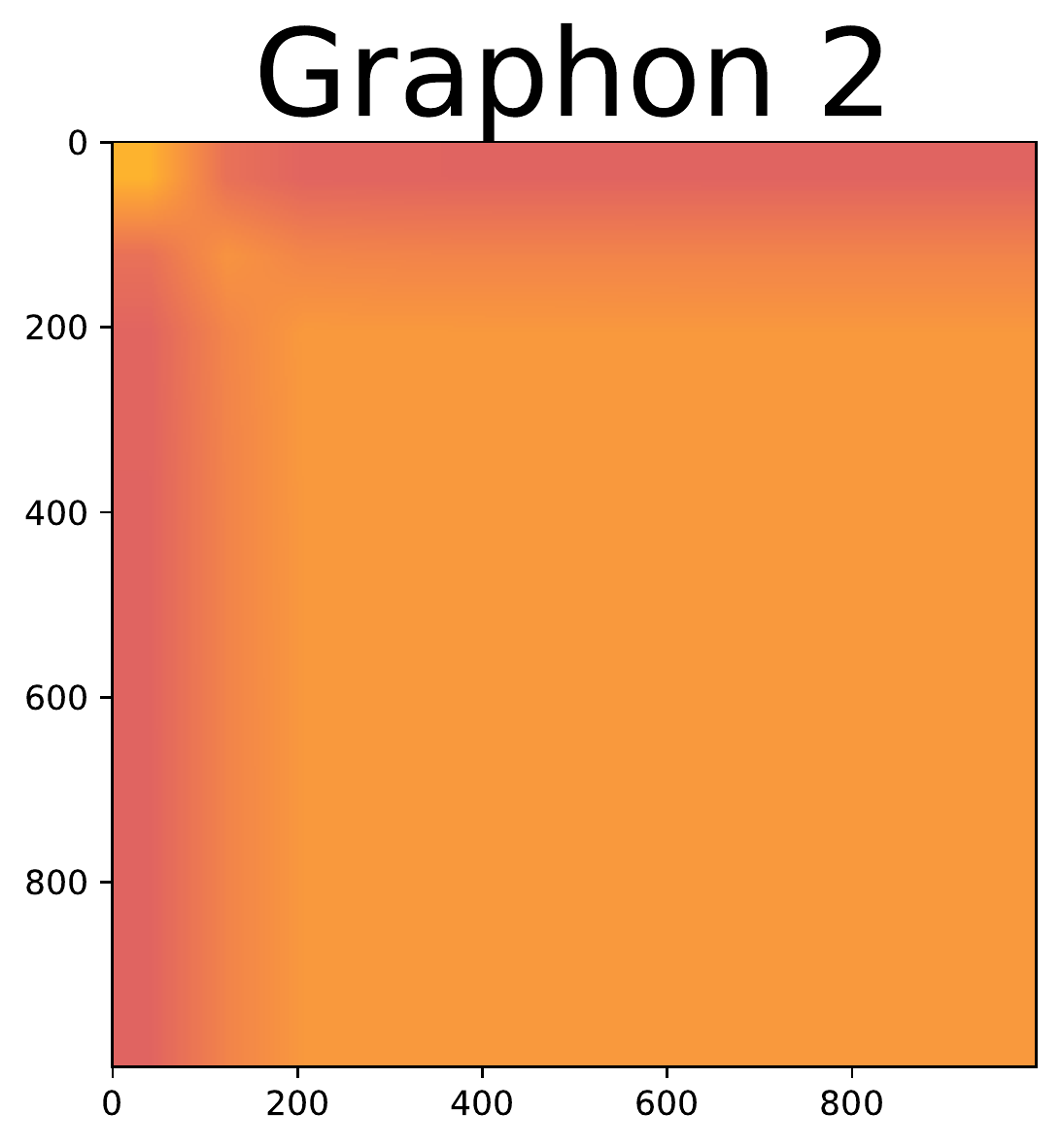}
    \includegraphics[height=1.5cm]{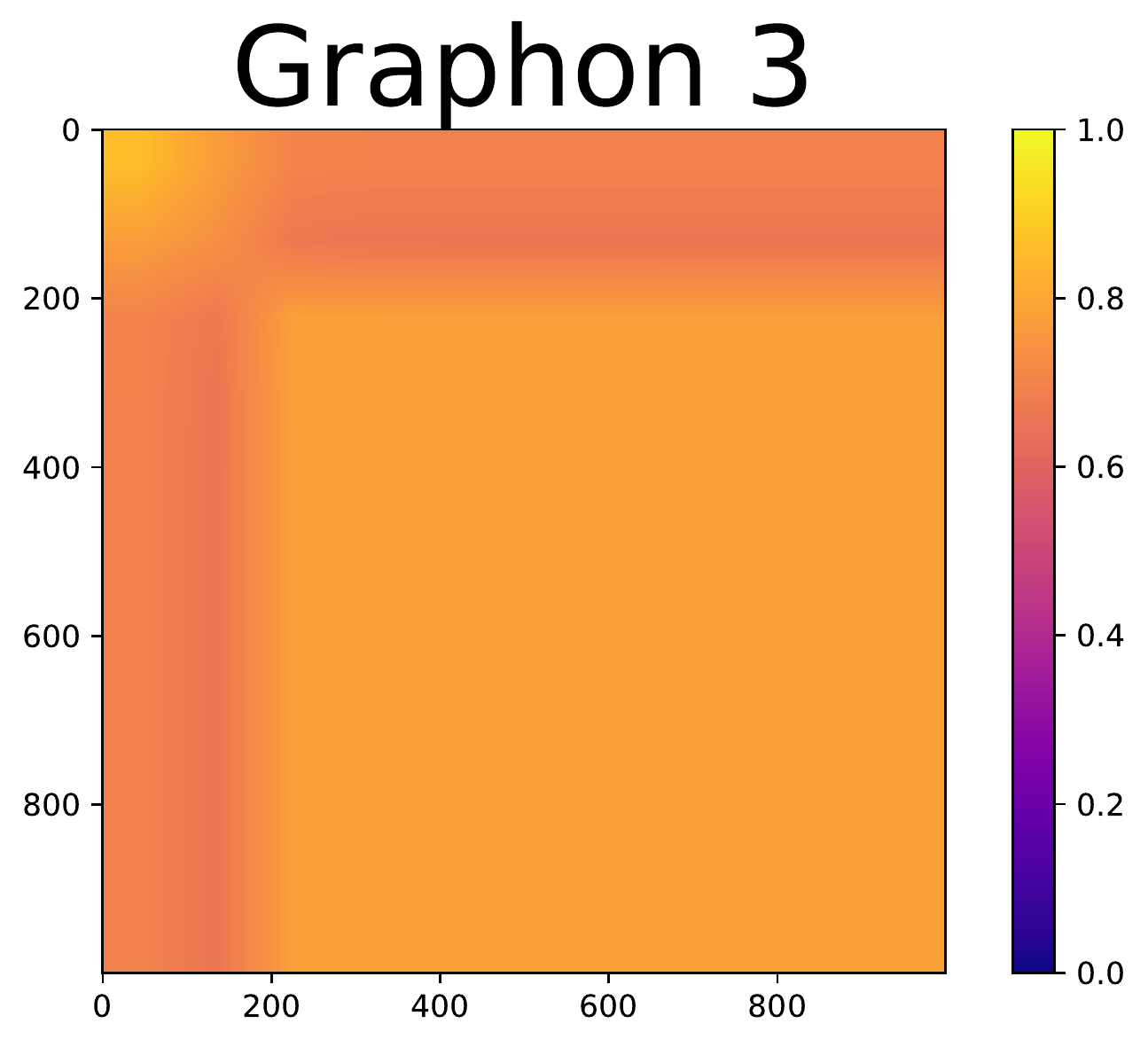}\label{fig:imdb2}
    }
    \vspace{-8pt}
    \caption{\small{Illustrations of estimated graphons.}}
    \label{fig:imdb}
\end{figure}

\section{Conclusions}
In this paper, we propose a novel method to learn graphons from unaligned graphs.
Our method minimizes an upper bound of the cut distance between the target graphons and their approximations, which leads to a GW barycenter problem. 
To extend our method to practical scenarios, we developed two structured variants of the basic GWB algorithm. 
In the future, we plan to improve the robustness of our method to the weight of the smoothness regularizer and further reduce the complexity of our method by applying the recursive GW distance~\cite{xu2019scalable} or the sliced GW distance~\cite{titouan2019sliced} to accelerate the computation of optimal transports. 
Additionally, because the graphon is naturally a generative graph model, we will consider using the model to achieve graph generation tasks.

\section{Acknowledgment} 
The Duke University component of this work was supported in part by DARPA, DOE, NIH, ONR and NSF, and a portion of the work performed by the first two authors was performed when they were affiliated with Duke. 
Hongteng Xu was supported in part by Beijing Outstanding Young Scientist Program (NO. BJJWZYJH012019100020098) and National Natural Science Foundation of China (No. 61832017).
Dixin Luo was supported in part by the Beijing Institute of Technology Research Fund Program for Young Scholars (XSQD-202107001) and the project 2020YFF0305200. 
We thank Thomas Needham and Samir Chowdhury for their constructive suggestions.
\bibstyle{aaai21}
\bibliography{graphon_gwb}

\end{document}